\documentclass{article}

\usepackage{arxiv}

\usepackage{natbib}

\usepackage{algorithm}
\usepackage{algorithmic}

\usepackage{hyperref}
\usepackage{url}

\newcommand\blfootnote[1]{%
  \begingroup
  \renewcommand\thefootnote{}\footnote{#1}%
  \addtocounter{footnote}{-1}%
  \endgroup
}

\title{Finite-Sum Optimization: A New Perspective for Convergence to a Global Solution}

\author{Lam M. Nguyen$^{1*}$, Trang H. Tran$^{2*}$, \textbf{Marten van Dijk}$^{3}$ \\
$^{1}$ IBM Research, Thomas J. Watson Research Center, Yorktown Heights, NY, USA \\
$^{2}$ School of Operations Research and Information Engineering, Cornell University, Ithaca, NY, USA\\
$^{3}$ CWI Amsterdam, The Netherlands \\
\\
\texttt{LamNguyen.MLTD@ibm.com}, \texttt{htt27@cornell.edu}, \texttt{marten.van.dijk@cwi.nl}
}

\usepackage{pifont}

\usepackage{epsfig}
\usepackage{amssymb}
\usepackage{amsmath}
\usepackage{amsthm}
\usepackage{amsfonts}
\usepackage{bbding}
\usepackage{array}
\usepackage{paralist}
\usepackage{xargs}                      
\usepackage{caption}

\newtheorem{theorem}{Theorem}
\newtheorem{lemma}{Lemma}

\newtheorem{corollary}{Corollary}

\newtheorem{definition}{Definition}
\newtheorem{remark}{Remark}
\newtheorem{assumption}{Assumption}

\newcolumntype{C}[1]{>{\centering\let\newline\\\arraybackslash\hspace{0pt}}m{#1}}

\newcommand\tagthis{\addtocounter{equation}{1}\tag{\theequation}}

\DeclareMathOperator{\Ocal}{\mathcal{O}}

\renewcommand{\top}{T}

\usepackage{xcolor}

\newcommand{\R}{\mathbb{R}}

\newcommand{\zero}[1]{{\boldsymbol{0}}}

\makeatletter
\newenvironment{framework}[1][hpt!]{%
    \renewcommand{\ALG@name}{Framework}
    \begin{algorithm}[#1]%
    }{\end{algorithm}
}
\makeatother

\begin{document}

\maketitle

\begin{abstract}
Deep neural networks (DNNs) have shown great success in many machine learning tasks. Their training is challenging since the loss surface of the network architecture is generally non-convex, or even non-smooth. How and under what assumptions is guaranteed convergence to a \textit{global} minimum possible? We propose a reformulation of the minimization problem allowing for a new recursive algorithmic framework. By using bounded style assumptions, we prove convergence to an $\varepsilon$-(global) minimum using $\mathcal{\tilde{O}}(1/\varepsilon^3)$ gradient computations. Our theoretical foundation motivates further study, implementation, and optimization of the new  algorithmic framework and further investigation of its non-standard bounded style assumptions. This new direction broadens our understanding of why and under what circumstances training of a DNN converges to a global minimum.
\end{abstract}

\blfootnote{$^{*}$ Equal contribution. Correspondence to: Lam M. Nguyen.}

\section{Introduction}\label{sec_intro}


In recent years, deep neural networks (DNNs) have shown a great success in many machine learning tasks. 
However, training these neural networks is challenging since the loss surface of network architecture is generally non-convex, or even non-smooth. 
Thus, there have been a long-standing question on how optimization algorithms may converge to a global minimum. 
Many previous work have investigated Gradient Descent algorithm and its stochastic version for over-parameterized setting \citep{pmlr-v80-arora18a, soudry2018implicit, pmlr-v97-allen-zhu19a, pmlr-v97-du19c,ZouG19nips}. 
Although these works have shown promising convergence results under certain assumptions, there is still a lack of new efficient methods that can guarantee convergence to a global solution for machine learning optimization. 
In this paper, we address this problem using a different perspective. Instead of analyzing the traditional finite-sum formulation, we adopt a new {\it composite formulation} that exactly depicts the structure of machine learning where a data set is used to learn a common classifier.



{\bf Representation.} Let $ \left\{( {x}^{(i)}, y^{(i)}) \right\}_{i=1}^n$ be a given training set with $ {x}^{(i)} \in \R^m, y^{(i)} \in \R^c$, we investigate the following novel representation for deep learning tasks: 
\begin{align*}
    \min_{w \in \mathbb{R}^d} \left\{ F(w) = \frac{1}{n} \sum_{i=1}^n \phi_i ( h ( w ; i) ) \right \}, \tagthis \label{new_loss_01}
\end{align*}
where $h(\cdot ; i): \mathbb{R}^d \to \mathbb{R}^c$, $i\in [n]=\{1,\ldots,n\}$, 
is the classifier for each input data $ {x}^{(i)}$;
and $\phi_i: \mathbb{R}^c \to \mathbb{R}$, $i\in [n]$, is the loss function corresponding to each output data $y^{(i)}$.
Our {\it composite formulation} \eqref{new_loss_01} is a special case of the finite-sum problem 
$\min_{w \in \mathbb{R}^d} \left\{ F(w) = \frac{1}{n}
\sum_{i=1}^n f(w; i) \right\}$
where each individual function $f(\cdot; i)$ is a composition of the loss function $\phi_i$ and the classifier $h(\cdot; i)$.
This problem covers various important applications in machine learning, including logistic regression and neural networks. 
The most common approach for the finite-sum problem is using first-order methods such as (stochastic) gradient algorithms and making assumptions on the component functions $f(\cdot; i)$. As an alternative, we further investigate the structure of the loss function $\phi_i$ and narrow our assumption on the classifier $h(\cdot; i)$. 
For the purpose of this work, we first consider convex and Lipschitz-smooth loss functions while the classifiers can be non-convex. 
Using this representation, we propose a new framework followed by two algorithms that guarantee convergence to a global solution for the minimization problem. 

{\bf Algorithmic Framework.}
Representation (\ref{new_loss_01}) admits  a new perspective. Our key insight is to (A) define $ z_i^{(t)}=h(w^{(t)};i)$, where $t$ is an iteration count of the outer loop in our algorithmic framework. Next (B), we want to approximate the change  $ z_i^{(t+1)}- z_i^{(t)}$ in terms of a step size times the gradient  
$$\nabla \phi_i( z_i^{(t)})= ( \partial \phi_i( z ) /\partial z_a )_{a \in [c]} \big |_{z=z_i^{(t)}},$$  
and (C) we approximate the change $h(w^{(t+1)};i)-h(w^{(t)};i)$ in terms of the first order derivative
$$H^{(t)}_i = ( \partial h_a(w ;i)/\partial w_b )_{a \in [c],b \in [d]} \big |_{w = w^{(t)}}.$$
Finally, we combine (A), (B), and (C) to equate the approximations of $ z_i^{(t+1)}- z_i^{(t)}$ and $h(w^{(t+1)};i)-h(w^{(t)};i)$.
This leads to a recurrence on $w^{(t)}$ of the form $w^{(t+1)}=w^{(t)}-\eta^{(t)} v^{(t)}$, where $\eta^{(t)}$ is a step size and which involves computing $v^{(t)}$ by solving a convex quadratic subproblem, see the details in Section~\ref{sec_new_framework}. 
We explain two methods for approximating a solution for the derived subproblem.
We show how to approximate the subproblem by transforming it into a strongly convex problem by adding a regularizer which can be solved in closed form. And we show how to use Gradient Descent (GD) on the subproblem to find an approximation $v^{(t)}$ of its solution.

{\bf Convergence Analysis.} Our analysis introduces non-standard bounded style assumptions. Intuitively, we assume that our convex and quadratic subproblem has a {\em bounded} solution.
This allows us to prove a total complexity of {$\tilde{\Ocal}(\frac{1}{\varepsilon^3})$} to find an $\varepsilon$-(global) solution that satisfies $F(\hat{w}) - F_* \leq \varepsilon$, where $F_*$ is the global minimizer of $F$.
Our analysis applies to a wide range of applications in machine learning: Our results hold for squared loss and softmax cross-entropy loss and applicable for a range of activation functions in DNN as we only assume that the $h(\cdot;i)$ are twice continuously differentiable 
and their Hessian matrices (second order derivatives) as well as their gradients (first order derivatives) are bounded. 

{\bf Contributions and Outline.}
Our contributions in this paper can be summarized as follows. 
\begin{itemize}
    \item 
    We propose a new representation (\ref{new_loss_01}) for analyzing the machine learning minimization problem.  Our formulation utilizes the structure of machine learning tasks where a training data set of inputs and outputs is used to learn a common classifier. Related work in Section \ref{subsec:related_work}  shows how (\ref{new_loss_01}) is different from the classical finite-sum problem. 
     
    \item 
    Based on the new representation we propose a novel algorithm framework. The algorithmic framework approximates a solution to a subproblem  for which we show two distinct approaches. 
    
    \item 
   For general DNNs and based on bounded style assumptions, we prove a total complexity of {$\tilde{\Ocal}(\frac{1}{\varepsilon^3})$}  to find an $\varepsilon$-(global) solution that satisfies $F(\hat{w}) - F_* \leq \varepsilon$, where $F_*$ is the global minimizer of $F$. 
\end{itemize}
We emphasize  that our focus is on developing a new theoretical foundation and that a  translation to a practical implementation with  empirical results is for future work.
Our theoretical foundation motivates  further study, implementation, and optimization of the new  algorithmic framework and further investigation of its non-standard bounded style assumptions. This new direction broadens our understanding of why and under what circumstances  training of a DNN converges to a global minimum.


The rest of this paper is organized as follows. Section \ref{subsec:related_work} discusses related work. Section \ref{sec_background} describes our setting and deep learning representation. Section \ref{sec_new_framework} explains our key insight and derives our Framework \ref{new_alg_framework}. Section \ref{sec_new_algs} presents our algorithms and their convergence to a global solution. All technical proofs are deferred to the Appendix. 





\section{Related Work}\label{subsec:related_work}
\textbf{Formulation for Machine Learning Problems.}
The finite-sum problem is one of the most important and fundamental problems in machine learning. 
Analyzing this model is the most popular approach in the machine learning literature and it has been studied intensively throughout the years \citep{bottou_survey,svrg_nonconvex-reddi16,2011duchi11a}.
Our new formulation \eqref{new_loss_01} is a special case of the finite-sum problem, however, it is much more complicated than the previous model since it involves the data index $i$ both inside the classifiers $h(\cdot;i)$ and the loss functions $\phi_i$. 
For a comparison, previous works only consider a common loss function $l (\hat{y}, y)$ for the predicted value $\hat{y}$ and output data $y$ \citep{zou2018stochastic, soudry2018implicit}. 
{Our modified version of loss function $\phi_i$ is a natural setting for machine learning. We note that when $h(w;i)$ is the output produced by a model, our goal is to match this output with the corresponding target $y^{(i)}$. For that reason, the loss function for each output has a dependence on the output data $y^{(i)}$, and is denoted by $\phi_i$. This fact reflects the natural setting of machine learning where the outputs are designed to fit different targets, and the optimization process depends on both outer function $\phi_i$ and inner functions $h(\cdot; i)$.}
This complication may potentially bring a challenge to theoretical analysis. However, with separate loss functions, we believe this model will help to exploit better the structure of machine learning problems and gain more insights on the neural network architecture. 

Other related composite optimization models are also investigated thoroughly in \citep{Lewis2016APM, zhang2019stochastic, pmlr-v119-tran-dinh20a}. {Our model is different from these works as it does not have a common function wrapping outside the finite-sum term, as in} \citep{Lewis2016APM}. {Note that a broad class of variance reduction algorithms} (e.g. SAG \citep{SAG}, SAGA \citep{SAGA}, SVRG \citep{SVRG}, SARAH \citep{Nguyen2017sarah}) {is designed specifically for the finite-sum formulation and is known to have certain benefits over Gradient Descent.}  In addition, the multilevel composite problem considered in \citep{zhang2021multilevel} also covers empirical risk minimization problem. However our formulation does not match their work since our inner function $h(w;i)$ is not an independent expectation over some data distribution, but a specific function that depends on the current data.

\textbf{Global Convergence for Neural Networks.}
A recent popular line of research is studying the dynamics of optimization methods on some specific neural network architectures.
There are some early works that show the global convergence of Gradient Descent (GD) for simple linear network and two-layer network \citep{brutzkus2017sgd,soudry2018implicit,pmlr-v97-arora19a,du2019gradient}.
Some further works extend these results to deep learning architectures \citep{pmlr-v97-allen-zhu19a,pmlr-v97-du19c,ZouG19nips}. These theoretical guarantees are generally proved for the case when the last output layer is fixed, which is not standard in practice. A recent work \citep{nguyen2020global} prove the global convergence for GD when all layers are trained with some initial conditions. 
However, these results are for neural networks without bias neurons and it is unclear how these analyses can be extended to handle the bias terms of deep networks with different activations. 
Our novel framework and algorithms do not exclude learning bias layers 
as in \citep{nguyen2020global}.

Using a different algorithm, \cite{brutzkus2017sgd} investigate Stochastic Gradient Descent (SGD) for two-layer networks in a restricted linearly separable data setting. 
This line of research continues with the works from \citep{pmlr-v97-allen-zhu19a,zou2018stochastic} and later with \citep{ZouG19nips}. They justify the global convergence of SGD for deep neural networks for some probability depending on the number of input data and the initialization process. 

\textbf{Over-Paramaterized Settings and other Assumptions for Machine Learning.}
Most of the modern learning architectures are over-parameterized, which means that the number of parameters are very large and often far more than the number of input data. 
    Some recent works prove the global convergence of Gradient Descent when the number of neurons are extensively large, e.g. \citep{ZouG19nips} requires $\Omega(n^8)$ neurons for every hidden layer, and \citep{nguyen2020global} improves this number to $\Omega(n^3)$. If the initial point satisfies some special conditions, then they can show a better dependence of  $\Omega(n)$. 
In \citep{pmlr-v97-allen-zhu19a}, {the authors
initialize the weights using a random Gaussian distribution where the variance depends on the dimension of the problem. In non-convex setting, they prove the convergence of SGD using the assumption that the dimension depends inversely on the tolerance $\epsilon$.}
We will discuss how these over-paramaterized settings might be a necessary condition to develop our theory. 

Other standard assumptions for machine learning include the bounded gradient assumption \citep{Nemirovski2009, pegasos, svrg_nonconvex-reddi16, smg_tran21b}. 
It is also common to assume all the iterations of an algorithm stay in a bounded domain \citep{2011duchi11a,2018Levy,2019gurbuzbalaban,reddi2019,vaswani2021adaptive}.
Since we are analyzing a new {\it composite formulation}, it is understandable that our assumptions may also not be standard. However, we believe that there is a strong connection between our assumptions and the traditional setting of machine learning. We will discuss this point more clearly in Section \ref{sec_new_framework}.

\section{Background}\label{sec_background}


In this section, we discuss our formulation and notations in detail. Although this paper focuses on deep neural networks, our framework and theoretical analysis are general and applicable for other learning architectures. 

\textbf{Deep Learning Representation.} 
Let $\{( {x}^{(i)},y^{(i)})\}_{i=1}^n$ be a training data set where $ {x}^{(i)} \in \mathbb{R}^m$ is a training input and $y^{(i)} \in \mathbb{R}^c$ is a training output. 
We consider a fully-connected neural network with $L$ layers, where the $l$-th layer, $l \in \{ 0,1,\dots,L\}$, has $n_l$ neurons. 
We represent layer $0$-th and $L$-th layer  as input and output layers, respectively, that is, $n_0 = d$ and $n_L = c$. For $l\in \{1,\ldots, L\}$, let $W^{(l)} \in \mathbb{R}^{n_{l-1} \times n_{l}}$ and $b^{(l)} \in \mathbb{R}^{n_l}$, where $\{ (W^{(l)}, b^{(l)})_{l=1}^L \}$ represent the parameters of the neural network. A classifier $h(w ; i)$ is formulated as
\begin{align*}
    & h(w ; i) & = W^{(L)\top} \sigma_{L-1} ( W^{(L-1)\top} \sigma_{L-2} ( \dots \sigma_1 (W^{(1)\top}  {x}^{(i)} + b^{(1)} ) \dots ) + b^{(L-1)}  )  + b^{(L)}, 
\end{align*}
where $w = \textbf{\text{vec}}(\{ W^{(1)} , b^{(1)} , \dots, W^{(L)} , b^{(L)} \}) \in \mathbb{R}^d$ is the vectorized weight and $\{\sigma_l\}_{l=1}^{L-1}$ are some activation functions. The most common choices for machine learning are ReLU, sigmoid, hyperbolic tangent and softplus. 
For $j\in [c]$,
$h_j(\cdot;i): \R^d \to \R$ denotes the component function of the output $h(\cdot;i)$, for each data $i \in [n]$ respectively. 
Moreover, we define 
$h^*_i = \arg \min_{z \in \mathbb{R}^c} \phi_i(z), i \in [n]$.

\textbf{Loss Functions. }
The well-known loss functions in neural networks for solving classification and regression problems are \textit{softmax cross-entropy loss} and \textit{square loss}, respectively:

{\em (Softmax) Cross-Entropy Loss:} $F(w)=\frac{1}{n}\sum_{i=1}^n f(w;i)$ with
\begin{align*}
    f (w; i) &= - y^{(i)\top} \log( \text{softmax} ( h(w ; i) ) ) .
    \tagthis \label{entropy_loss_f}
\end{align*}
    
    {\em Squared Loss}: $F(w)=\frac{1}{n}\sum_{i=1}^n f(w;i)$ with
    \begin{align*}
        f(w ; i) = \frac{1}{2} \| h (w ; i) - y^{(i)} \|^2. \tagthis \label{squared_loss_F}
    \end{align*}





We provide some basic definitions in optimization theory to support our theory. 
\begin{definition}[$L$-smooth]
\label{defn_smooth}
Function $\phi: \mathbb{R}^c \to \mathbb{R}$ is $L_{\phi}$-smooth if there exists a constant $L_{\phi} > 0$ such that, $\forall x_1,x_2 \in \mathbb{R}^c$, 
\begin{align*}
\| \nabla \phi(x_1) - \nabla \phi(x_2) \| \leq L_{\phi} \| x_1 - x_2 \|. \tagthis\label{eq:Lsmooth_basic}
\end{align*} 
\end{definition}


\begin{definition}[Convex] \label{defn_convex}
Function $\phi: \mathbb{R}^c \to \mathbb{R}$ is convex if    
 $\forall x_1, x_2 \in \mathbb{R}^c$,
\begin{gather*}
\phi(x_1)  - \phi(x_2) \geq \langle \nabla \phi(x_2), x_1 - x_2 \rangle. \tagthis\label{eq:convex_02}
\end{gather*}
\end{definition}

The following corollary shows the properties of softmax cross-entropy loss \eqref{entropy_loss_f} and squared loss  \eqref{squared_loss_F}.

\begin{corollary}\label{cor_slce_phi_convex}
For softmax cross-entropy loss  \eqref{entropy_loss_f} and squared loss  \eqref{squared_loss_F},
there exist functions $h(\cdot ; i): \mathbb{R}^d \to \mathbb{R}^c$ and $\phi_i: \mathbb{R}^c \to \mathbb{R}$ such that, for  $i\in [n]$, $\phi_i(z)$ is convex and $L_{\phi}$-smooth with $L_{\phi} = 1$, and
\begin{align*}
    f ( w ; i ) = \phi_i ( h ( w  ; i) ) = \phi_i ( z ) \big |_{z = h ( w ; i )}.
    \tagthis \label{eq_component_loss_01}
\end{align*}
\end{corollary}

The following lemma is a standard result in \citep{nesterov2004}. 
\begin{lemma}[\citep{nesterov2004}]\label{lem_smooth_convex}
If $\phi$ is $L_{\phi}$-smooth and convex, then for $\forall z \in \mathbb{R}^c$,
\begin{align*}
    \| \nabla \phi (z ) \|^2 & \leq 2 L_{\phi} ( \phi ( z ) - \phi (z_*) ), \tagthis \label{eq_smooth_convex}
\end{align*}
where $z_* = \arg\min_{z} \phi(z)$. 
\end{lemma}
The following useful derivations can be used later in our theoretical analysis. 
Since $\phi_i$ is convex, by Definition~\ref{defn_convex} we have
\begin{align*}
    \phi_{i} (h(w ; i)) & \geq \phi_{i} (h(w' ; i)) + \left \langle \nabla_z \phi_{i} (z) \Big |_{z = h(w' ; i)} ,  h(w ; i) - h(w' ; i)  \right \rangle. \tagthis \label{eq_phi_convex}
\end{align*}

If $\phi_i$ is convex and $L_{\phi}$-smooth, then by Lemma~\ref{lem_smooth_convex}
\begin{align*}
    \left\| \nabla_z \phi_{i} (z) \Big |_{z = h(w ; i)}  \right \|^2 &\leq 2 L_{\phi} \left [  \phi_{i} (h(w ; i)) - \phi_{i} ( h^*_i )  \right ],  \tagthis \label{eq_phi_convex_smooth}
\end{align*}
where $h^*_i = \arg \min_{z \in \mathbb{R}^c} \phi_i(z)$. 

We compute gradients of $f(w;i)$ in terms of $\phi_i ( h ( w  ; i) )$. 
\begin{itemize}
    \item \textbf{Gradient of softmax cross-entropy loss:}
    \begin{align*}
        & \nabla \phi_i ( z ) \big |_{z = h ( w ; i )} = \left( \frac{\partial \phi_i ( z )}{\partial z_1} \Big |_{z = h ( w ; i )} , \dots , \frac{\partial \phi_i ( z )}{\partial z_c} \Big |_{z = h ( w ; i )} \right)^\top,
    \end{align*}
    where for $j \in [c]$, $\frac{\partial \phi_i ( z )}{\partial z_j} \Big |_{z = h ( w ; i )}$ is
    \begin{align*}
        \begin{cases}
        \frac{\exp \left( [ h(w ; i) ]_j - [ h(w; i) ]_{I(y^{(i)})}   \right)}{\sum_{k=1}^c \exp \left( [ h(w ; i) ]_k - [ h(w; i) ]_{I(y^{(i)})}   \right)} \ , \ j \neq I(y^{(i)}) \\
        - \frac{\sum_{k \neq I(y^{(i)})} \exp \left( [ h(w ; i) ]_k - [ h(w; i) ]_{I(y^{(i)})}   \right)}{\sum_{k=1}^c \exp \left( [ h(w ; i) ]_k - [ h(w; i) ]_{I(y^{(i)})}   \right)} \ , \ j = I(y^{(i)})
        \end{cases}. \tagthis \label{entropy_gradient_phi}
    \end{align*}
    \item \textbf{Gradient of squared loss:}
    \begin{align*}
        \nabla \phi_i ( z ) \big |_{z = h ( w ; i )} &= h (w ; i) - y^{(i)}. \tagthis \label{squared_gradient_phi}
    \end{align*}
\end{itemize}

We introduce the notations that we use throughout the paper in Table~\ref{tab:notations}. 

\begin{table*}[]
\caption{Table of notations}
\label{tab:notations}
\renewcommand{\arraystretch}{1.25}
\begin{center}
    \begin{tabular}{|c l|} 
    \hline
    \textbf{Notation} & \textbf{Meaning}  \\ 
    \hline
    $F_*$ & Global minimization function of $F$ in \eqref{new_loss_01} \\ 
    & $F_* = \min_{w \in \mathbb{R}^d} F(w)$ \\ [1ex] 
    \hline
    $h^*_i$ & $h^*_i = \arg \min_{z \in \mathbb{R}^c} \phi_i(z)$, $i \in [n]$ \\ [1ex] 
    \hline
    $v_{*}^{(t)}$ & Solution of the convex problem in \eqref{opt_prob_01}\\
    & $\min_{v \in \mathbb{R}^d} \frac{1}{2} \frac{1}{n} \sum_{i=1}^{n} \| \eta^{(t)} H_i^{(t)} v - \alpha_i^{(t)} \nabla_z \phi_i (h(w^{(t)};i))  \|^2$ \\ [1ex] 
    \hline
    $v^{(t)}$ & An approximation of $v_{*}^{(t)}$ which is used as the search direction in Framework~\ref{new_alg_framework}~~~~~~~~~ \\
    \hline
    $\hat{v}_{*\varepsilon }^{(t)}$ & A vector that satisfies \\
    & $ \frac{1}{2} \frac{1}{n} \sum_{i=1}^{n} \| \eta^{(t)} H_i^{(t)} v - \alpha_i^{(t)} \nabla_z \phi_i (h(w^{(t)};i))  \|^2\leq \varepsilon^2$ 
    \\ [1ex] 
    &for some $\varepsilon > 0$ and $\| \hat{v}_{*\varepsilon}^{(t)} \|^2 \leq V$, for some $V > 0$. 
    \\
    \hline
     $v_{*\ \text{reg}}^{(t)}$ & Solution of the strongly convex problem in \eqref{opt_prob_L2}\\
    & $\min_{v \in \mathbb{R}^d} \left\{\frac{1}{2} \frac{1}{n} \sum_{i=1}^{n} \| \eta^{(t)} H_i^{(t)} v - \alpha_i^{(t)} \nabla_z \phi_i (h(w^{(t)};i))  \|^2 + \frac{\varepsilon^2}{2} \|v\|^2 \right\}$ \\ [1ex] 
    \hline
    \end{tabular}
\end{center}
\end{table*}

\section{New Algorithm Framework}\label{sec_new_framework}

\subsection{Key Insight}\label{subsec_motivation}



We assume $f ( w ; i ) = \phi_i ( h ( w  ; i) )$ with $\phi_i$  convex and $L_{\phi}$-smooth. 
%
Our goal is to utilize the convexity of the outer function $\phi_i$. 
In order to simplify notation, we write $\nabla_z \phi_i (h(w^{(t)};i))$ instead of $\nabla_z \phi_i (z) \big |_{z = h(w^{(t)};i)}$ and denote $z_i^{(t)} = h(w^{(t)};i)$.
Starting from the current weight $w^{(t)}$, we would like to find the next point $w^{(t+1)}$ that satisfies the following approximation {\em for all $i \in [n]$}: 
\begin{align*}
    h(w^{(t+1)};i) & = z_i^{(t+1)} \approx z_i^{(t)} - \alpha_i^{(t)} \nabla_z \phi_i (z_i^{(t)}) = h(w^{(t)};i) - \alpha_i^{(t)} \nabla_z \phi_i (h(w^{(t)};i)). \  \tagthis \label{update_x}
\end{align*}


{We can see that} this approximation is a ``noisy" version of a gradient descent update for every function $\phi_i$, simultaneously for all $i \in [n]$. 
In order to do this, we use the following update
\begin{align*}
    w^{(t+1)} = w^{(t)} - \eta^{(t)} v^{(t)}, \tagthis \label{update_w}
\end{align*}
where $\eta^{(t)} > 0$ is a learning rate and $v^{(t)}$ is a search direction that helps us approximate equation \eqref{update_x}.
If the update term $\eta^{(t)} v^{(t)}$ is small enough, and if $h(\cdot; i)$ has some nice smooth properties, then from basic calculus we have the following approximation: 
\begin{align*}
    h(w^{(t+1)};i) &= h(w^{(t)} - \eta^{(t)} v^{(t)} ;i) \approx h(w^{(t)};i) - H_i^{(t)} \big(\eta^{(t)} v^{(t)}\big),\tagthis \label{eq_002.1}
\end{align*}
where $H_i^{(t)}$ is a  matrix in $\R^{c \times d}$ with first-order derivatives.
Motivated by approximations \eqref{update_x} and \eqref{eq_002.1}, we consider the following optimization problem: 
\begin{align*}
    v_*^{(t)} &= \arg \min_{v \in \mathbb{R}^d} \left\{ \frac{1}{2} \frac{1}{n} \sum_{i=1}^{n} \| H_i^{(t)} \big( \eta^{(t)} v\big) - \alpha_i^{(t)} \nabla_z \phi_i (h(w^{(t)};i))  \|^2 \right\} \tagthis \label{opt_prob_01}. 
\end{align*}

Hence, by solving for the solution $v_*^{(t)}$ of problem \eqref{opt_prob_01} we are able to find a search direction for the key approximation \eqref{update_x}. This yields our new algorithmic Framework 
\ref{new_alg_framework}, see below. 



\begin{framework}[hpt!]
   \caption{New Algorithm Framework}
   \label{new_alg_framework}
\begin{algorithmic}
   \STATE {\bfseries Initialization:} Choose an initial point $w^{(0)}\in\R^d$;
   \FOR{$t=0,1,\cdots,T-1 $}
   \STATE Solve for an approximation $v^{(t)}$ of the solution $v_*^{(t)}$ of the problem in \eqref{opt_prob_01}
   \begin{align*}
    v_*^{(t)} &= \arg \min_{v \in \mathbb{R}^d} \left\{ \frac{1}{2} \frac{1}{n} \sum_{i=1}^{n} \| \eta^{(t)} H_i^{(t)} v - \alpha_i^{(t)} \nabla_z \phi_i (h(w^{(t)};i))  \|^2 \right\}.       
   \end{align*}
   \STATE Update $w^{(t+1)} = w^{(t)} - \eta^{(t)} v^{(t)}$
   \ENDFOR
\end{algorithmic}
\end{framework}

\subsection{Technical Assumptions}\label{subsec_assumptions}

\begin{assumption}\label{ass_phi}
The loss function $\phi_i$ is convex and $L_{\phi}$-smooth for $i \in [n]$. Moreover, 
we assume that it is lower bounded, i.e. $\inf_{z \in \R^c}\phi_i (z) > -\infty$ for $i \in [n]$.
\end{assumption}
We have shown the convexity and smoothness of squared loss and softmax cross-entropy loss in Section \ref{sec_background}.
The bounded property of $\phi_i$ is required in any algorithm for the well-definedness of \eqref{new_loss_01}.
Now, in order to use the Taylor series approximation, 
we need the following assumption on the neural network architecture $h$:

\begin{assumption}\label{ass_approx_h}
We assume that $h(\cdot;i)$ is twice continuously differentiable for all $i \in [n]$ (i.e. the second-order partial derivatives of all scalars $h_j(\cdot;i)$ are continuous for all $j\in[c]$ and $i \in [n]$), and that their Hessian matrices are bounded, that is, there exists a $G > 0$ such that for all $w \in \R^d$, $i \in [n]$  and  $j \in [c]$,
\begin{align*}
    \|M_{i, j} (w) \|  = \left \|\textbf{J}_w \left(\nabla_w h_j(w;i) \right)\right\|\leq G, \ 
    \tagthis \label{bounded_hessian}
\end{align*}
where $\textbf{J}_w$ denotes the Jacobian\footnote{For a continuously differentiable function $g(w): \R^d \to \R^c$ we define the Jacobian
 $\textbf{J}_w(g(w))$ as the matrix $(\partial g_a(w)/\partial w_b)_{a\in [c],b\in [d]}$.}.
\end{assumption}


\begin{remark}[Relation to second-order methods]\label{rem_ass1}
 {Although our analysis requires an assumption on the Hessian matrices of $h(w;i)$, our algorithms do not use any second order information or try to approximate this information. 
Our theoretical analysis focused on the approximation of the classifier and the gradient information, therefore is not related to the second order type algorithms. 
It is currently unclear how to apply second order methods into our problem, however, this is an interesting research question to expand the scope of this work.}
\end{remark}



Assumption \ref{ass_approx_h} allows us to apply a Taylor approximation of each function $h_j(\cdot; i)$ with which we  prove the following Lemma that bounds the error in equation \eqref{eq_002.1}:

\begin{lemma}\label{lem_approx_h}
Suppose that Assumption \ref{ass_approx_h} holds for the classifier $h$. Then for all $i \in [n]$ and $0 \leq t < T$, 
\begin{align*}
    h(w^{(t+1)};i) & = h(w^{(t)} - \eta^{(t)} v^{(t)} ;i) = h(w^{(t)};i) - \eta^{(t)} H_i^{(t)} v^{(t)} + \epsilon_i^{(t)}, \tagthis \label{eq_002.3}
\end{align*}
where 
\begin{equation}
    H_i^{(t)}=\textbf{J}_{w} (h ( w ; i ))|_{w = w^{(t)}}\in \mathbb{R}^{c \times d} \label{eq_define_Hit}
\end{equation} is defined as  the Jacobian matrix of $h(w; i)$  at $w^{(t)}$ 
and  entries $\epsilon_{i,j}^{(t)}$, $j\in [c]$, of vector $\epsilon_i^{(t)}$ satisfy
\begin{align*}
    |\epsilon_{i,j}^{(t)}| \leq\frac{1}{2}(\eta^{(t)})^2 \| v^{(t)} \|^2 G.
    \tagthis \label{eq_epsilon}
\end{align*}
\end{lemma}

In order to {approximate \eqref{update_x} combined with \eqref{eq_002.1}, that is, to make sure the right hand sides of \eqref{update_x} and \eqref{eq_002.1} are close to one another}, we consider the optimization problem (\ref{opt_prob_01}): 
\begin{align*}
    v_*^{(t)} &= \arg \min_{v \in \mathbb{R}^d} \left\{ \frac{1}{2} \frac{1}{n} \sum_{i=1}^{n} \| \eta^{(t)} H_i^{(t)} v - \alpha_i^{(t)} \nabla_z \phi_i (h(w^{(t)};i))  \|^2  \right\}. 
\end{align*}
The optimal value of problem \eqref{opt_prob_01} is equal to 0 if there exists a vector ${v}_*^{(t)}$ satisfying $\eta^{(t)} H_i^{(t)} {v}_*^{(t)} = \alpha_i^{(t)} \nabla_z \phi_i (h(w^{(t)};i)) $ for every $i \in [n]$. 
Since the solution ${v}_*^{(t)} $ is in $ \R^d$ and $\nabla_z \phi_i (h(w^{(t)};i))$ is in $\R^c$, this condition is equivalent to a linear system with $n\cdot c$ constraints and $d$ variables. In the over-parameterized setting where dimension $d$ is sufficiently large  ($d \gg n\cdot c$) and there are no identical data, there exists almost surely a vector ${v}_*^{(t)}$ that interpolates all the training set, see the Appendix for details.

Let us note that an approximation of ${v}_*^{(t)}$ serves as the search direction for Framework \ref{new_alg_framework}. For this reason, the solution ${v}_*^{(t)}$ of problem \eqref{opt_prob_01} plays a similar role as a gradient in the search direction of (stochastic) gradient descent method. 
It is standard to assume a bounded gradient in the machine learning literature \citep{Nemirovski2009, pegasos, svrg_nonconvex-reddi16}.
Motivated by these facts, we assume the following Assumption \ref{ass_solve_v}, which implies the existence of a near-optimal {\em bounded} solution 
of (\ref{opt_prob_01}):

\begin{assumption}\label{ass_solve_v}
 {We consider an over-parameterized setting where dimension $d$ is sufficiently large enough to interpolate all the data and the tolerance $\varepsilon$.} We assume that 
there exists a bound $V>0$ such that for $\varepsilon>0$ 
and $0 \leq t < T$ as in Framework \ref{new_alg_framework},
there exists a vector $\hat{v}_{*\varepsilon}^{(t)}$ with $\| \hat{v}_{*\varepsilon}^{(t)} \|^2 \leq V$ so that
\begin{align*}
    \frac{1}{2} \frac{1}{n} \sum_{i=1}^{n} \| \eta^{(t)} H_i^{(t)}\hat{v}_{*\varepsilon}^{(t)} - \alpha_i^{(t)} \nabla_z \phi_i (h(w^{(t)};i))  \|^2 \leq \varepsilon^2.
\end{align*}
\end{assumption}

Our Assumption \ref{ass_solve_v}  {requires a nice dependency on the tolerance $\varepsilon$ for the gradient matrices $H_i^{(t)}$ and $\nabla_z \phi_i (h(w^{(t)};i))$. We note that at the starting point $t=0$, these matrices may depend on $\varepsilon$ due to the initialization process and the dependence of $d$ on $\varepsilon$. This setting is similar to previous works, e.g.} \citep{pmlr-v97-allen-zhu19a}.
In the Appendix, we show an example of neural network architecture where Assumption \ref{ass_solve_v} is justified at the start of the training process. 

\section{New Algorithms and Convergence Results}\label{sec_new_algs}
\subsection{Approximating the solution using regularizer}
Since problem \eqref{opt_prob_01} is convex and quadratic, we consider the following regularized problem: 
\begin{align*}
    \min_{v \in \mathbb{R}^d}  \left\{\Psi^{(t)}(v) = \Phi^{(t)}(v)  + \frac{\varepsilon^2}{2} \| v \|^2\right\}, \tagthis \label{opt_prob_L2}
\end{align*}
where
\begin{align*}
    \Phi^{(t)}(v) = \frac{1}{2} \frac{1}{n} \sum_{i=1}^{n} \| \eta^{(t)} H_i^{(t)} v - \alpha_i^{(t)} \nabla_z \phi_i (h(w^{(t)};i))  \|^2. 
\end{align*}


for some small $\varepsilon > 0$ and $t \geq 0$. 
It is widely known that problem~\eqref{opt_prob_L2} is strongly convex, and has a unique minimizer $v_{*\ \text{reg}}^{(t)}$. The global minimizer satisfies $\nabla_v \Psi^{(t)} ( v_{*\ \text{reg}}^{(t)} ) = 0$.
We have
\begin{align*}
    \nabla_v \Psi^{(t)} ( v ) 
    &= \frac{1}{n} \sum_{i=1}^{n} \Big[ \eta^{(t)} H_i^{(t)}{^\top}  H_i^{(t)} \eta^{(t)} v - \alpha_i^{(t)} \eta^{(t)} H_i^{(t)}{^\top} \nabla_z \phi_i (h(w^{(t)};i)) \Big] + \varepsilon^2 \cdot v \\
    &= \left( \frac{1}{n} \sum_{i=1}^{n}  \eta^{(t)} H_i^{(t)}{^\top}  H_i^{(t)} \eta^{(t)} + \varepsilon^2 I \right) v - \left( \frac{1}{n} \sum_{i=1}^{n} \alpha_i^{(t)} \eta^{(t)} H_i^{(t)}{^\top} \nabla_z \phi_i (h(w^{(t)};i)) \right). 
\end{align*}
Therefore,
\begin{align*}
    v_{*\ \text{reg}}^{(t)} & = \left( \frac{1}{n} \sum_{i=1}^{n}  \eta^{(t)} H_i^{(t)}{^\top}  H_i^{(t)} \eta^{(t)} + \varepsilon^2 I \right)^{-1} \left( \frac{1}{n} \sum_{i=1}^{n} \alpha_i^{(t)} \eta^{(t)} H_i^{(t)}{^\top} \nabla_z \phi_i (h(w^{(t)};i)) \right). \tagthis \label{closed_form_v_opt}
\end{align*}


If $\varepsilon^2$ is small enough, then $v_{*\ \text{reg}}^{(t)}$ is a close approximation of the solution $v_{*}^{(t)}$ for problem \eqref{opt_prob_01}. 
Our first algorithm updates Framework \ref{new_alg_framework} based on this approximation.

\renewcommand{\thealgorithm}{1}
\begin{algorithm}[hpt!]
   \caption{Solve for the exact solution of the regularized problem}
   \label{alg_closed_form}
\begin{algorithmic}
   \STATE {\bfseries Initialization:} Choose an initial point $w^{(0)}\in\R^d$, tolerance $\varepsilon > 0$;
   \FOR{$t=0,1,\cdots,T-1 $}
   \STATE Update the search direction $v^{(t)}$ as the solution $v_{*\ \text{reg}}^{(t)}$ of problem in \eqref{opt_prob_L2}:
   \begin{align*}
       v^{(t)} = v_{*\ \text{reg}}^{(t)} &= \left( \frac{1}{n} \sum_{i=1}^{n}  \eta^{(t)} H_i^{(t)}{^\top}  H_i^{(t)} \eta^{(t)} + \varepsilon^2 I \right)^{-1}  \left( \frac{1}{n} \sum_{i=1}^{n} \alpha_i^{(t)} \eta^{(t)} H_i^{(t)}{^\top} \nabla_z \phi_i (h(w^{(t)};i)) \right)
   \end{align*}
   \STATE Update $w^{(t+1)} = w^{(t)} - \eta^{(t)} v^{(t)}$
   \ENDFOR
\end{algorithmic}
\end{algorithm}

The following Lemma shows the relation between the regularized solution $v_{*\ \text{reg}}^{(t)}$ and the optimal solution of the original convex problem $\hat{v}_{*\varepsilon}^{(t)}$. 

\begin{lemma}\label{lem_stronglyconvex_prob}
For given $\varepsilon>0$,
suppose that Assumption~\ref{ass_solve_v} holds for bound $V>0$. Then, for iteration $0 \leq t < T$, the optimal solution $v_{*\ \text{reg}}^{(t)}$ of  problem  \eqref{opt_prob_L2} satisfies $\| v_{*\ \text{reg}}^{(t)} \|^2 \leq 2 + V$
and
\begin{align*}
    & \frac{1}{2} \frac{1}{n} \sum_{i=1}^{n} \| \eta^{(t)} H_i^{(t)} v_{*\ \text{reg}}^{(t)} - \alpha_i^{(t)} \nabla_z \phi_i (h(w^{(t)};i))  \|^2  \leq ( 1 + \frac{V}{2} ) \varepsilon^2. \tagthis \label{eq_lem_stronglyconvex_prob}
\end{align*}
\end{lemma}


Based on Lemma \ref{lem_stronglyconvex_prob}, we  guarantee the convergence to a global solution of Algorithm \ref{alg_closed_form} and prove our first theorem. 
Since it is currently expensive to solve for the exact solution of problem \eqref{opt_prob_L2}, our algorithm serves as a theoretical method to obtain the convergence to a global solution for the finite-sum minimization.

\begin{theorem}\label{thm_main_result_closed_form}
Let $w^{(t)}$ be generated by Algorithm \ref{alg_closed_form} where we use the closed form solution for the search direction.
We execute Algorithm \ref{alg_closed_form} for $T = \frac{\beta}{\varepsilon}$ outer loops for some constant $\beta>0$.
We assume Assumption \ref{ass_phi} holds.
Suppose that Assumption \ref{ass_approx_h} holds for $G > 0$ and Assumption \ref{ass_solve_v} holds for $V > 0$.
We set the step size equal to $\eta^{(t)} = D\sqrt{\varepsilon}$ for some $D > 0$ and choose a learning rate 
$ \alpha_i^{(t)} = ( 1 + \varepsilon ) \alpha_i^{(t-1)} = ( 1 + \varepsilon )^t \alpha_i^{(0)}$.  Based on $\beta$, we define 
$\alpha_i^{(0)} = \frac{\alpha}{e^{\beta} L_{\phi}}$
with $\alpha \in (0, \frac{1}{3})$. 
Let $F_*$ be the global minimizer of $F$, and $h^*_i = \arg \min_{z \in \mathbb{R}^c} \phi_i(z), i \in [n]$. Then
\allowdisplaybreaks
\begin{align*}
    \frac{1}{T} \sum_{t=0}^{T-1} [ F (w^{(t)}) - F_* ] & \leq   \frac{e^{\beta} L_{\phi}( 1 + \varepsilon )}{2 (1-3\alpha) \alpha \beta} \cdot  \frac{1}{n} \sum_{i=1}^n  \| h(w^{(0)};i) - h_i^* \|^2 \cdot \varepsilon \\
    & \quad + \frac{e^{\beta} L_{\phi}(3\varepsilon+2)}{8\alpha(1-3\alpha)}  \left[ c (4+(V+2) G D^2)^2  + 8+ 4V\right] \cdot \varepsilon. \tagthis \label{eq_thm_main_result_closed_form_new}
\end{align*}
\end{theorem}
 {We note that $\beta$ is a constant for the purpose of choosing the number of iterations $T$.  The analysis can be simplified by choosing $\beta = 1$ with $T = \frac{1}{\varepsilon}$. Notice that the common convergence criteria for finding a stationary point for non-convex problems is $ \frac{1}{T} \sum_{t=1}^{T} || \nabla F ( w_t ) ||^2 \leq O(\varepsilon)$. This criteria has been widely used in the existing literature for non-convex optimization problems. 
Our convergence criteria $\frac{1}{T} \sum_{t=1}^{T} [ F (w_t) - F_* ] \leq O(\varepsilon)$ is slightly different, in order to find a global solution for non-convex problems.}

Our proof for Theorem \ref{thm_main_result_closed_form} is novel and insightful. It is originally motivated by the Gradient Descent update \eqref{update_x} and the convexity of the loss functions $\phi_i$. For this reason it may not be a surprise that 
Algorithm \ref{alg_closed_form} can find an $\varepsilon$-global solution after $ \Ocal\left(\frac{1}{\varepsilon}\right)$ iterations. However, computing the exact solution in every iteration might be extremely challenging, {especially when the number of samples $n$ is large}. Therefore, we present a different approach to this problem in the following section.

\subsection{Approximation using Gradient Descent}

In this section, we use Gradient Descent (GD) algorithm to solve the strongly convex problem \eqref{opt_prob_L2}.
It is well-known that if $\psi(x) - \frac{\mu}{2} \| x \|^2$ is convex for $\forall x \in \mathbb{R}^c$, then $\psi(x)$ is $\mu$-strongly convex (see e.g. \citep{nesterov2004}). Hence $\Psi(\cdot)$ is $\varepsilon^2$-strongly convex.
For each iteration $t$, we use GD to find a search direction $v^{(t)}$ which is sufficiently close to the optimal solution $v_{*\ \text{reg}}^{(t)}$ in that
\begin{align*}
     \| v^{(t)} - v_{*\ \text{reg}}^{(t)} \| \leq \varepsilon. \tagthis \label{criteria_eps_sol.3}
\end{align*}
Our Algorithm \ref{alg_GD_solved.2} is described as follows.

\renewcommand{\thealgorithm}{2}
\begin{algorithm}[hpt!]
   \caption{Solve the regularized problem using Gradient Descent}
   \label{alg_GD_solved.2}
\begin{algorithmic}
   \STATE {\bfseries Initialization:} Choose an initial point $w^{(0)}\in\R^d$, tolerance $\varepsilon > 0$;
   \FOR{$t=0,1,\cdots,T-1 $}
   \STATE Use Gradient Descent algorithm to solve Problem \eqref{opt_prob_L2} and find a solution $v^{(t)}$ that satisfies 
   $$\| v^{(t)} - v_{*\ \text{reg}}^{(t)} \| \leq \varepsilon.$$
   \STATE Update $w^{(t+1)} = w^{(t)} - \eta^{(t)} v^{(t)}$
   \ENDFOR
\end{algorithmic}
\end{algorithm}

Since Algorithm \ref{alg_GD_solved.2} can only approximate a solution within some $\varepsilon$-preciseness, we need a supplemental assumption for the analysis of our next Theorem \ref{thm_main_result_GD_solved.3}: 
\begin{assumption}\label{ass_bounded_hessian.2}
Let $H_i^{(t)}$ be the Jacobian matrix defined in Lemma~\ref{lem_approx_h}.
 We assume that there exists some constant $H > 0$ such that, for $i\in [n]$, $\varepsilon > 0$, and $0 \leq t < T$ as in Algorithm \ref{alg_GD_solved.2},
\begin{align*}
    \| H_i^{(t)} \| \leq \frac{H}{\sqrt{\varepsilon}}. \tagthis \label{eq_ass_bounded_hessian.2}
\end{align*}
\end{assumption}
Assumption \ref{ass_bounded_hessian.2}  {requires a mild condition on the bounded Jacobian of $h(w;i)$, and the upper bound may depend on $\varepsilon$. This flexibility allows us to accommodate a good dependence of $\varepsilon$ for the theoretical analysis.}
We are now ready to present our convergence theorem for Algorithm \ref{alg_GD_solved.2}.
\begin{theorem}\label{thm_main_result_GD_solved.3}
Let $w^{(t)}$ be generated by Algorithm \ref{alg_GD_solved.2}
where $v^{(t)}$ satisfies \eqref{criteria_eps_sol.3}.
We execute Algorithm \ref{alg_GD_solved.2} for $T = \frac{\beta}{\varepsilon}$ outer loops for some constant $\beta>0$.
We assume Assumption \ref{ass_phi} holds.
Suppose that Assumption \ref{ass_approx_h} holds for $G > 0$, Assumption \ref{ass_solve_v} holds for $V > 0$
and Assumption \ref{ass_bounded_hessian.2} holds for $H > 0$. 
We set the step size equal to $\eta^{(t)} = D\sqrt{\varepsilon}$ for some $D > 0$ and choose a learning rate 
$ \alpha_i^{(t)} = ( 1 + \varepsilon ) \alpha_i^{(t-1)} = ( 1 + \varepsilon )^t \alpha_i^{(0)}$. Based on $\beta$, we define 
$\alpha_i^{(0)} = \frac{\alpha}{e^{\beta} L_{\phi}}$
with $\alpha \in (0, \frac{1}{4})$. 
Let $F_*$ be the global minimizer of $F$, and $h^*_i = \arg \min_{z \in \mathbb{R}^c} \phi_i(z), i \in [n]$.
Then
\allowdisplaybreaks
\begin{align*}
    \frac{1}{T} \sum_{t=0}^{T-1} [ F (w^{(t)}) - F_* ] & \leq  \frac{e^{\beta} L_{\phi}( 1 + \varepsilon )}{2 ( 1 - 4 \alpha) \alpha \beta} \cdot  \frac{1}{n} \sum_{i=1}^n  \| h(w^{(0)};i) - h_i^* \|^2 \cdot \varepsilon 
     \\ & \quad + \frac{e^{\beta} L_{\phi}(4\varepsilon+3)}{2\alpha( 1 - 4 \alpha)}  \left[ D^2H^2 + c (2+(V+\varepsilon^2 +2) G D^2)^2 + 2+ V  \right] \cdot \varepsilon.  
\end{align*}
\end{theorem}

Theorem \ref{thm_main_result_GD_solved.3} implies Corollary \ref{cor_thm_main_result_GD_solved_new.3} which provides the computational complexity for Algorithm \ref{alg_GD_solved.2}. {Note that for (Stochastic) Gradient Descent, we derive the complexity in terms of component gradient calculations 
for the finite-sum problem \eqref{new_loss_01}. As an alternative, for Algorithm \ref{alg_GD_solved.2} we compare the number of component gradients in problem \eqref{opt_prob_L2} where $\Phi^{(t)}(v) = \frac{1}{n}\sum_{i=1}^{n} \psi_i^{(t)} ( v )$. Such individual gradient has the following form: 
\begin{align*}
    \nabla_v \psi_i^{(t)} ( v ) 
    &=  \eta^{(t)} H_i^{(t)}{^\top}  H_i^{(t)} \eta^{(t)} v  - \alpha_i^{(t)} \eta^{(t)} H_i^{(t)}{^\top} \nabla_z \phi_i (h(w^{(t)};i)).
\end{align*}


In machine learning applications, the gradient of $f(\cdot;i)$ is calculated using automatic differentiation (i.e.  backpropagation). Since $f(\cdot;i)$ is the composition of the network structure $h(\cdot; i)$ and loss function $\phi_i(\cdot)$, this process also computes the Jacobian matrix $H_i^{(t)}$ and the gradient $\nabla_z \phi_i (h(w^{(t)};i))$ at a specific weight $w^{(t)}$. Since matrix-vector multiplication computation is not expensive, the cost for computing the component gradient of problem \eqref{opt_prob_L2} is similar to problem \eqref{new_loss_01}.
}

\begin{corollary}\label{cor_thm_main_result_GD_solved_new.3}
Suppose that the conditions in Theorem~\ref{thm_main_result_GD_solved.3} hold with $\eta^{(t)} = \frac{D \sqrt{\hat{\varepsilon}}}{\sqrt{N}}$ for some $D > 0$ and $0 < \hat{\varepsilon} \leq N$ (that is, we set $\varepsilon=\hat{\varepsilon}/N$), where 
\begin{align*}
    N & = \frac{e^{\beta} L_{\phi}\sum_{i=1}^n  \| h(w^{(0)};i) - h_i^* \|^2}{n( 1 - 4 \alpha) \alpha \beta}   + \frac{7 e^{\beta} L_{\phi}\left[D^2H^2 + c (2+(V+3) G D^2)^2 + 2+ V\right]}{2\alpha( 1 - 4 \alpha)} . 
\end{align*} 
Then, the total complexity to guarantee
\begin{align*}
    \min_{0 \leq t \leq T-1} [ F (w^{(t)}) - F_* ] \leq \frac{1}{T} \sum_{t=0}^{T-1} [ F (w^{(t)}) - F_* ] \leq \hat{\varepsilon}
\end{align*}
is  {$\mathcal{O}\left(n \frac{N^3 \beta}{\hat{\varepsilon}^3} (D^2H^2 + (\hat{\varepsilon}^2/N))\log(\frac{N}{\hat{\varepsilon}}) \right)$}. 
\end{corollary}


\begin{remark}
Corollary \ref{cor_thm_main_result_GD_solved_new.3} shows
that $\Ocal\left(1/\hat{\varepsilon} \right)$ outer loop iterations are needed in order to reach an $\hat{\varepsilon}$-global solution, and it proves that each iteration needs the equivalent of  { $\mathcal{O}\left(\frac{n}{\hat{\varepsilon}^2} \log(\frac{1}{\hat{\varepsilon}}) \right)$}  gradient computations for computing an approximate solution. In total,  Algorithm \ref{alg_GD_solved.2} has total complexity  { $\mathcal{O}\left(\frac{n}{\hat{\varepsilon}^3} \log(\frac{1}{\hat{\varepsilon}}) \right)$}  for finding an $\hat{\varepsilon}$-global solution. 



For a comparison, Stochastic Gradient Descent uses a total of $\Ocal(\frac{1}{\varepsilon^2})$ gradient computations to find a stationary point satisfying $\mathbb{E}[\| \nabla F ( \hat{w} ) \|^2] \leq \varepsilon$ for non-convex problems \citep{ghadimi2013stochastic}.
Gradient Descent has a better complexity in terms of $\varepsilon$, i.e. $\Ocal(\frac{n}{\varepsilon})$ such that $\| \nabla F ( \hat{w} ) \|^2 \leq \varepsilon$ \citep{nesterov2004}. However, both methods may not be able to reach a global solution of (\ref{new_loss_01}). In order to guarantee global convergence for nonconvex settings, one may resort to use Polyak-Lojasiewicz (PL) inequality \citep{polyak_condition,pmlr-v130-gower21a}. This assumption is widely known to be strong, which implies that every stationary point is also a global minimizer.





\end{remark}


\section{Further Discussion and Conclusions}\label{sec_conclusion}


This paper presents an alternative {\it composite formulation} for solving the finite-sum optimization problem. Our formulation 
allows a new way of exploiting the structure of machine learning problems and the convexity of squared loss and softmax cross entropy loss, and leads to
a novel algorithmic framework that guarantees convergence to a global solution (when the outer loss functions are convex and Lipschitz-smooth). 

Our analysis is general and can be applied to various different learning architectures, in particular, our analysis and assumptions match practical neural networks; in recent years, there has been a great interest in the structure of deep learning architectures for over-parameterized settings \citep{pmlr-v80-arora18a, pmlr-v97-allen-zhu19a,nguyen2020global}. Algorithm \ref{alg_GD_solved.2} demonstrates a gradient method
 to solve the regularized problem, however, other methods can be applied to our framework (e.g. conjugate gradient descent). 

Our theoretical foundation motivates  further study, implementation, and optimization of the new  algorithmic framework and further investigation of its non-standard bounded style assumptions. 
Possible research directions include more practical algorithm designs based on our Framework \ref{new_alg_framework}, and different related methods to solve the regularized problem and approximate the solution such as Stochastic Gradient Descent and its stochastic first-order variants (e.g. \citep{2011duchi11a,KingmaB14,bottou_survey,Nguyen2018_sgdhogwild,Nguyen2019_sgd_new_aspects,nguyen2020unified}). This potentially leads to a new class of efficient algorithms for machine learning problems. This paper presents a new perspective to the research community.





\appendix

\section*{\Large Appendix}\label{sec_appendix}

\section{Useful Results}\label{sec_existing}

The following lemmas provide key tools for our results. 
\begin{lemma}[Squared loss]\label{lem_squared_convex_smooth}
Let $b \in \mathbb{R}^c$ and define $\phi(z) = \frac{1}{2}\| z - b \|^2$ for 
$z \in \mathbb{R}^c$. Then $\phi$ is convex and $L_{\phi}$-smooth  with $L_{\phi} = 1$. 
\end{lemma}

\begin{lemma}[Softmax cross-entropy loss]\label{lem_softmax_cross_entropy_convex_smooth}
Let index $a\in [c]$ and define
\begin{align*}
    \phi(z) = \log \left[\sum_{k=1}^{c} \exp ( z_k - z_a )  \right] = \log \left[ \sum_{k=1}^{c} \exp ( w_k^\top z )  \right],
\end{align*}
for 
$z = (z_1,\dots,z_c)^\top \in \mathbb{R}^c$, where 
$w_k = e_k - e_a$ with $e_i$ representing the $i$-th unit vector (containing $1$ at the $i$-th position and $0$ elsewhere).
Then $\phi$ is convex and $L_{\phi}$-smooth  with $L_{\phi} = 1$. 
\end{lemma}

\section{Additional Discussion}\label{sec_discuss}

\subsection{About Assumption \ref{ass_approx_h}}\label{subsec_discuss_ass1}

We make a formal assumption for the case $h(\cdot;i)$ is closely approximated by $k(\cdot;i)$. 
\begin{assumption}\label{ass_approx_h.2}

We assume that for all $i \in [n]$ there exists some approximations $k(w;i): \R^d \to \R^c$ such that
\begin{align*}
    |k_{j}(w;i) - h_j (w;i) | \leq \varepsilon, \ \forall w \in \R^d,\ i \in [n] \text{ and } j \in [c], \tagthis \label{eq_approx_eps}
\end{align*}

where $k(\cdot;i)$ are twice continuously differentiable (i.e. the second-order partial derivatives of all scalars $k_j(\cdot;i)$ are continuous for all $i \in [n]$), and that their Hessian matrices are bounded: 
\begin{align*}
    \|M_{i, j} (w) \|  = \left \|\textbf{J}_w \left(\nabla_w k_j(w;i) \right)\right\|\leq G, \ \forall w \in \R^d,\ i \in [n] \text{ and } j \in [c].\tagthis \label{bounded_hessian.3}
\end{align*}

\end{assumption}

Assumption \ref{ass_approx_h.2} allows us to prove the following Lemma that bound the error in equation \eqref{eq_002.1}:

\begin{lemma}\label{lem_approx_h.2}
Suppose that Assumption \ref{ass_approx_h.2} holds for the classifier $h$. Then for all $i \in [n]$ and $0 \leq t < T$, we have:
\begin{align*}
    h(w^{(t+1)};i) = h(w^{(t)} - \eta^{(t)} v^{(t)} ;i) = h(w^{(t)};i) - \eta^{(t)} H_i^{(t)} v^{(t)} + \epsilon_i^{(t)}, \tagthis \label{eq_002.4}
\end{align*}
where $H_i^{(t)}$ is defined to be the Jacobian matrix of 
the approximation $k(w; i)$  at $w^{(t)}$: 
\begin{align}
    H_i^{(t)} := \textbf{J}_{w} k ( w ; i )|_{w = w^{(t)}} = \begin{bmatrix}
    \frac{\partial k_1(w ; i)}{\partial w_1} & \dots & \frac{\partial k_1(w ; i)}{\partial w_d} \\ 
    \dots & \dots & \dots \\ 
    \frac{\partial k_c(w ; i)}{\partial w_1} & \dots & \frac{\partial k_c(w ; i)}{\partial w_d}
    \end{bmatrix} \Bigg|_{w = w^{(t)}}
\in \mathbb{R}^{c \times d}. \label{eq_define_Hit.2}
\end{align}
Additionally we have,
\begin{align*}
    |\epsilon_{i,j}^{(t)}| \leq\frac{1}{2}(\eta^{(t)})^2 \| v^{(t)} \|^2 G + 2\varepsilon, \ j \in [c]. \tagthis \label{eq_epsilon.2}
\end{align*}
\end{lemma}
Note that these result recover the case when $h(\cdot;i)$ is itself smooth. Hence we analyze our algorithms using the result of Lemma \ref{lem_approx_h.2}, which generalizes the result from Lemma \ref{lem_approx_h}.

\subsection{About Assumption \ref{ass_solve_v}}\label{subsec_discuss_ass2}
In this section, we justify the existence of the search direction in Assumption \ref{ass_solve_v} (almost surely). We argue that there exists a vector $\hat{v}_{*\varepsilon}^{(t)}$ satisfying
\begin{align*}
    \frac{1}{2} \frac{1}{n} \sum_{i=1}^{n} \| \eta^{(t)} H_i^{(t)}\hat{v}_{*\varepsilon}^{(t)} - \alpha_i^{(t)} \nabla_z \phi_i (h(w^{(t)};i))  \|^2 \leq \varepsilon^2.
\end{align*}
It is sufficient to find a vector $v$ satisfying that 
\begin{align*}
    \eta^{(t)} H_i^{(t)} v  = \alpha_i^{(t)} \nabla_z \phi_i (h(w^{(t)};i))   \text{ for every } i \in [n].
\end{align*}
Since the solution $v$ is in $ \R^d$ and $\nabla_z \phi_i (h(w^{(t)};i))$ is in $\R^c$, this condition is equivalent to a linear system with $n\cdot c$ constraints and $d$ variables.  
Let $A$ and $b$ be the following stacked matrix and vector: 
\begin{align*}
    A = \begin{bmatrix}
    H_1^{(t)} \eta^{(t)} \\ 
    \dots  \\ 
    H_n^{(t)} \eta^{(t)}
    \end{bmatrix} 
\in \mathbb{R}^{n \cdot c \times d}, \text{ and } b = \begin{bmatrix}
    \alpha_1^{(t)} \nabla_z \phi_1 (h(w^{(t)};i))   \\ 
    \dots  \\ 
    \alpha_n^{(t)} \nabla_z \phi_n (h(w^{(t)};i)) 
    \end{bmatrix} 
\in \mathbb{R}^{n \cdot c},
\end{align*}
then the problem reduce to finding the solution of the equation $Av = b$.
In the over-parameterized setting where dimension $d$ is sufficiently large  ($d \gg n\cdot c$), then rank $A = n \cdot c$ almost surely and there exists almost surely a vector $v$ that interpolates all the training set.

To demonstrate this fact easier, we consider a simple neural network where the classifier $h(w ; i)$ is formulated as
\begin{align*}
    h(w ; i) =  W^{(2)\top} \sigma ( W^{(1)\top} x^{(i)}),  
\end{align*}
where $c=1$, $W^{(1)} \in \R^{m \times l}$ and  $W^{(2)} \in \R^{l \times 1}$, $w = \textbf{\text{vec}}(\{ W^{(1)},W^{(2)} \}) \in \mathbb{R}^d$ is the vectorized weight where $d= l(m+1)$ and $\sigma $ is sigmoid activation function.


 $H_i^{(t)}$ is defined to be the Jacobian matrix of  $h(w; i)$  at $w^{(t)}$: 
\begin{align*}
    H_i^{(t)} := \textbf{J}_{w} h ( w ; i )|_{w = w^{(t)}} = \begin{bmatrix}
    \frac{\partial h(w ; i)}{\partial w_1} & \dots & \frac{\partial h(w ; i)}{\partial w_d} 
    \end{bmatrix} \Bigg|_{w = w^{(t)}}
\in \mathbb{R}^{1 \times d},
\end{align*}
then 
\begin{align*}
    A =\eta^{(t)} \begin{bmatrix}
    H_1^{(t)}  \\ 
    \dots  \\ 
    H_n^{(t)} 
    \end{bmatrix} 
    =\eta^{(t)} \begin{bmatrix}
   \frac{\partial h(w ; 1)}{\partial w_1} & \dots & \frac{\partial h(w ; 1)}{\partial w_d} \\
     \dots & \dots & \dots \\ 
    \frac{\partial h(w ; n)}{\partial w_1} & \dots & \frac{\partial h(w ; n)}{\partial w_d}
    \end{bmatrix} 
\in \mathbb{R}^{n \times d}.
\end{align*}

We want to show that $A$ has full rank, almost surely. 
We consider the over-parameterized setting where the last layer has at least $n$ neuron (i.e. $l=n$ and the simple version when $c = 1$. 
We argue that rank of matrix $A$ is greater than or equal to rank of the submatrix $B$ created by the weights of the last layer $W^{(2)} \in \R^{n}$:
\begin{align*}
    B =\begin{bmatrix}
   \frac{\partial h(w ; 1)}{\partial W^{(2)}_1} & \dots & \frac{\partial h(w ; 1)}{\partial W^{(2)}_{n}} \\ 
    \dots & \dots & \dots \\ 
    \frac{\partial h(w ; n)}{\partial W^{(2)}_1} & \dots & \frac{\partial h_1(w ; n)}{\partial W^{(2)}_{n}}
    \end{bmatrix} 
\in \mathbb{R}^{n  \times n }. 
\end{align*}

Note that $h(\cdot, i)$ is a linear function of the last weight layers (in this simple case $W^{(2)} \in \R^{n}$ and $\sigma ( W^{(1)\top} x^{(i)}) \in \R^n$), we can compute the partial derivatives as follows: 
\begin{align*}
    \frac{\partial h(w ; i)}{\partial W^{(2)}} = \sigma ( W^{(1)\top} x^{(i)}) ; \ i \in [n].
\end{align*}
Hence 
\begin{align*}
    B = \begin{bmatrix}
   \sigma ( W^{(1)\top} x^{(1)}) \\ 
    \dots  \\ 
    \sigma ( W^{(1)\top} x^{(n)})
    \end{bmatrix} \in \R^{n \times n}. 
\end{align*}
Assuming that there are no identical data, and $\sigma$ is the sigmoid activation, the set of weights $W^{(1)}$ that make matrix $B$ degenerate has measure zero. 
Hence $B$ has full rank almost surely, and we have the same conclusion for $A$. Therefore we are able to prove the almost surely existence of a solution $v$ of the linear equation $Av= b$ for simple two layers network. 
Using the same argument, this result can be generalized for larger neural networks where the dimension $d$ is sufficiently large ($d \gg nc$).

\subsection{Initialization example}
Our Assumption \ref{ass_solve_v} requires a nice dependency on the tolerance $\varepsilon$ for the gradient matrices $H_i^{(0)}$ and $\nabla_z \phi_i (h(w^{(0)};i))$. We note that at the starting point $t=0$, these matrices may depend on $\varepsilon$ due to the initialization process and the dependence of $d$ on $\varepsilon$. In order to accommodate the choice of learning rate $\eta^{(0)} = D\sqrt{\varepsilon}$ in our 
theorems, in this section we describe a network initialization that satisfies $\|H_i^{(0)}\| = \Theta\left(\frac{1}{\sqrt{\varepsilon}}\right)$ where the gradient norm $\| \nabla_z \phi_i (h(w^{(0)};i))\|$ is at most constant order with respect to $\varepsilon$. To simplify the problem, we only consider small-dimension data and networks without activation.

\textbf{About the target vector:} We choose $\phi_i$ to be the softmax cross-entropy loss. By Lemma \ref{lem_bounded_gradient_entropy} (see below), we have that the gradient norm is upper bounded by a constant $c$, where $c$ is the output dimension of the problem and is not dependent on $\varepsilon$. Note that when we stack all gradients for $n$ data points, then the size of new vector is still not dependent on $\varepsilon$.

\textbf{About the network architecture:}
For simplicity, we consider the following classification problem where

\begin{itemize}
    \item The input data is in $\R^2$. There are only two data points $\{ x^{(1)}, x^{(2)}\}$. Input data is bounded and non-degenerate (we will clarify this property later).
    \item The output data is (categorical) in $\mathbb{R}^2$: $\{y^{(1)} = (1,0), y^{(2)} = (0,1)\}$.
\end{itemize}
We want to have an over-parameterized setting where the dimension of weight vector is at least $nc = 4$. We consider a simple network with two layers, no biases and no activation functions. Let the number of neurons in the hidden layer be $m$. The flow of this network is $ \text{(in) } \R^2 \to \R^m \to \R^2 \text{ (out)}$. First, we consider the case where $m=1$.
\begin{itemize}
    \item The first layer has 2 parameters $(w_1, w_2)$ and only 1 neuron that outputs  $z^{(i)} = w_1 x_1^{(i)} + w_2 x_2^{(i)}$ (the subscript is for the coordinate of input data $x^{(i)}$). 
    \item The second layer has 2 parameters $(w_3, w_4)$. The final output is 
$$h(w, i) = [w_3 (w_1 x_1^{(i)} + w_2 x_2^{(i)}), w_4 (w_1 x_1^{(i)} + w_2 x_2^{(i)})]^\top \in \mathbb{R}^2,
$$
\end{itemize}

with $w = [w_1, w_2, w_3, w_4]^\top \in \mathbb{R}^4$. This network satisfies that the Hessian matrices of $h(w;i)$ are bounded. 
Let $Q$ and $b$ be the following stacked matrix and vector: 
\begin{align*}
    Q = \begin{bmatrix}
    H_1^{(0)} \\
    H_2^{(0)} 
    \end{bmatrix} 
\in \mathbb{R}^{4 \times 4}, \text{ and } b = \begin{bmatrix}
    \nabla_z \phi_1 (h(w^{(0)};1))   \\ 
    \nabla_z \phi_2 (h(w^{(0)};2)) 
    \end{bmatrix} 
\in \mathbb{R}^{4},
\end{align*}

Then we have the following: 
\begin{align*}
    Q = Q(w) &= \begin{bmatrix}
    H_1^{(0)} \\
    H_2^{(0)} 
    \end{bmatrix} = \begin{bmatrix}
    \nabla_w [w_3 (w_1 x_1^{(1)} + w_2 x_2^{(1)})]\\
    \nabla_w [w_4 (w_1 x_1^{(1)} + w_2 x_2^{(1)})]\\
    \nabla_w [w_3 (w_1 x_1^{(2)} + w_2 x_2^{(2)})]\\
    \nabla_w [w_4 (w_1 x_1^{(2)} + w_2 x_2^{(2)})]\\
    \end{bmatrix}\\
    &= \begin{bmatrix}
    w_3 x_1^{(1)} &w_3 x_2^{(1)} &w_1 x_1^{(1)} + w_2 x_2^{(1)} &0 \\
    w_4 x_1^{(1)} &w_4 x_2^{(1)} &0 &w_1 x_1^{(1)} + w_2 x_2^{(1)} \\
    w_3 x_1^{(2)} &w_3 x_2^{(2)} &w_1 x_1^{(2)} + w_2 x_2^{(2)} &0 \\
    w_4 x_1^{(2)} &w_4 x_2^{(2)} &0 &w_1 x_1^{(2)} + w_2 x_2^{(2)} \\
    \end{bmatrix}.
\end{align*}

The determinant of this matrix is a polynomial of the weight $w$ and the input data. 
Under some mild non-degenerate condition of the input data, we can choose some base point $w'$ that made this matrix invertible (note that if this condition is not satisfied, we can rescale/add a very small noise to the data - which is the common procedure in machine learning). 

Hence the system $Qu = b$ always has a solution. Now we consider the following two initializations:

1. We choose to initialize the starting point at $w^{(0)} = \frac{1}{\sqrt{\varepsilon}} w'$ and note that $Q(w)$ is a linear function of $w$ and $Q(w')$ is independent of $\varepsilon$. Then the norm of matrix $Q(w^{(0)})$ has the same scale with $\frac{1}{\sqrt{\varepsilon}}$.

2. 
Instead of choosing $m = 1$, we consider an over-parameterized network where $m = \frac{1}{\varepsilon}$ (recall that $m$ is the number of neurons in the hidden layer). 
The hidden layer in this case is: 
\begin{align*}
    z = \begin{cases}
    z_{1}^{(i)} &= w_{1,1}^{(1)} x_{1}^{(i)} + w_{2,1}^{(1)} x_{2}^{(i)} \\
    & \dots \\
    z_{m}^{(i)} &= w_{1,m}^{(1)} x_{1}^{(i)} + w_{2,m}^{(1)} x_{2}^{(i)}
    \end{cases}. 
\end{align*}
The output layer is:
\begin{align*}
    \begin{cases}
        y_{1}^{(i)} = z_{1}^{(i)} w_{1,1}^{(2)} + \dots + z_{m}^{(i)} w_{m,1}^{(2)} = ( w_{1,1}^{(1)} x_{1}^{(i)} + w_{2,1}^{(1)} x_{2}^{(i)} ) w_{1,1}^{(2)} + \dots + (w_{1,m}^{(1)} x_{1}^{(i)} + w_{2,m}^{(1)} x_{2}^{(i)}) w_{m,1}^{(2)}   \\
        y_{2}^{(i)} = z_{1}^{(i)} w_{1,2}^{(2)} + \dots + z_{m}^{(i)} w_{m,2}^{(2)} = ( w_{1,1}^{(1)} x_{1}^{(i)} + w_{2,1}^{(1)} x_{2}^{(i)} ) w_{1,2}^{(2)} + \dots + (w_{1,m}^{(1)} x_{1}^{(i)} + w_{2,m}^{(1)} x_{2}^{(i)}) w_{m,2}^{(2)}
    \end{cases}
\end{align*}

with $w = [w_{1,1}^{(1)},\dots,w_{1,m}^{(1)},w_{2,1}^{(1)},\dots,w_{2,m}^{(1)}, w_{1,1}^{(2)}, w_{1,2}^{(2)}, \dots, w_{m,1}^{(2)}, w_{m,2}^{(2)}]^\top \in \mathbb{R}^{4m}$.

Hence, 
\begin{align*}
    Q(w) 
    = \begin{bmatrix} 
    w_{1,1}^{(2)}  x_{1}^{(1)} &\dots \ w_{m,1}^{(2)}  x_{1}^{(1)} & w_{1,1}^{(2)}  x_{2}^{(1)} &\dots& w_{m,1}^{(2)}  x_{2}^{(1)} & z_{1}^{(1)} & 0 & \dots \ z_{m}^{(1)} & 0 \\
    w_{1,2}^{(2)}  x_{1}^{(1)} &\dots \ w_{m,2}^{(2)}  x_{1}^{(1)} & w_{1,2}^{(2)}  x_{2}^{(1)} &\dots& w_{m,2}^{(2)}  x_{2}^{(1)} & 0 & z_{1}^{(1)} & \dots \ 0 & z_{m}^{(1)} \\
    w_{1,1}^{(2)}  x_{1}^{(2)} &\dots \ w_{m,1}^{(2)}  x_{1}^{(2)} & w_{1,1}^{(2)}  x_{2}^{(2)} &\dots& w_{m,1}^{(2)}  x_{2}^{(2)} & z_{1}^{(2)} & 0 & \dots \ z_{m}^{(2)} & 0 \\
    w_{1,2}^{(2)}  x_{1}^{(2)} &\dots \ w_{m,2}^{(2)}  x_{1}^{(2)} & w_{1,2}^{(2)}  x_{2}^{(2)} &\dots& w_{m,2}^{(2)}  x_{2}^{(2)} & 0 & z_{1}^{(2)} & \dots \ 0 & z_{m}^{(2)}
    \end{bmatrix}. 
\end{align*}

Hence, the number of (possibly) non-zero elements in each row is $3m = \frac{3}{\varepsilon}$. 

For matrix $A$ of rank $r$, we have $\| A \|_2 \leq \| A \|_F \leq \sqrt{r} \| A \|_2$. Since the rank of $Q(w)$ is at most 4 ($nc = 4$, independent of $\varepsilon$), we only need to find the Frobenius norm of $Q(w)$. We have
\begin{align*}
    \|Q(w)\|_F = \sqrt{\sum_{i=1}^{4} \sum_{j=1}^{4m} | q_{ij} |^2 }. 
\end{align*}

Let $q_{min}$ and $q_{max}$ be the element with smallest/largest magnitude of $Q(w)$. Suppose that $x^{(i)} \neq (0,0)$  and choose $w \neq 0$ such that $z \neq 0$, $q_{min} > 0$ and independent of $\varepsilon$. Hence, $\frac{\sqrt{8}}{\sqrt{\varepsilon}} | q_{min} | \leq \| Q(w) \|_F \leq  \frac{\sqrt{12}}{\sqrt{\varepsilon}} | q_{max} |$. 

Hence, $\| Q( w ) \| = \Theta \left( \frac{1}{\sqrt{\varepsilon}} \right)$. Therefore this simple network initialization supports the dependence on $\varepsilon$ for our Assumption \ref{ass_solve_v}. We note that a similar setting is found in \citep{pmlr-v97-allen-zhu19a}, where the authors
initialize the weights using a random Gaussian distribution with a variance depending on the dimension of the problem. In non-convex setting, they prove the convergence of SGD using the assumption that the number of neurons $m$ depends inversely on the tolerance $\varepsilon$.



\begin{lemma}\label{lem_bounded_gradient_entropy}
For softmax cross-entropy loss, and $x = h(w;i) \in \mathbb{R}^c$, for $\forall w \in \mathbb{R}^d$ and $i \in [n]$, we have
\begin{align*}
    \left\| \nabla_z \phi_{i} (x) \Big |_{x = h(w ; i)}  \right \|^2 \leq c. \tagthis \label{bounded_gradient_entropy}
\end{align*}
\end{lemma}
\begin{proof}
By \eqref{entropy_gradient_phi}, we have for $i = 1,\dots, n$, 
\begin{itemize}
    \item For $j \neq I(y^{(i)})$: 
    \begin{align*}
        \left( \frac{\partial \phi_i ( x )}{\partial x_j} \Big |_{x = h ( w ; i )}  \right)^2 & = \left( \frac{\exp \left( [ h(w ; i) ]_j - [ h(w; i) ]_{I(y^{(i)})}   \right)}{\sum_{k=1}^c \exp \left( [ h(w ; i) ]_k - [ h(w; i) ]_{I(y^{(i)})}   \right)} \right)^2 \\
        & = \left( \frac{\exp \left( [ h(w ; i) ]_j - [ h(w; i) ]_{I(y^{(i)})}   \right)}{1 + \sum_{k \neq I(y^{(i)})} \exp \left( [ h(w ; i) ]_k - [ h(w; i) ]_{I(y^{(i)})}   \right)} \right)^2 \leq 1. 
    \end{align*}
    \item For $j = I(y^{(i)})$:
    \begin{align*}
        \left( \frac{\partial \phi_i ( x )}{\partial x_j} \Big |_{x = h ( w ; i )}  \right)^2 & = \left(  \frac{\sum_{k \neq I(y^{(i)})} \exp \left( [ h(w ; i) ]_k - [ h(w; i) ]_{I(y^{(i)})}   \right)}{\sum_{k=1}^c \exp \left( [ h(w ; i) ]_k - [ h(w; i) ]_{I(y^{(i)})}   \right)}  \right)^2 \\
        & = \left(  \frac{\sum_{k \neq I(y^{(i)})} \exp \left( [ h(w ; i) ]_k - [ h(w; i) ]_{I(y^{(i)})}   \right)}{1 + \sum_{k \neq I(y^{(i)})} \exp \left( [ h(w ; i) ]_k - [ h(w; i) ]_{I(y^{(i)})}   \right)}  \right)^2 \leq 1
    \end{align*}
\end{itemize}
Hence, for $i = 1,\dots, n$, 
\begin{align*}
    \left\| \nabla_z \phi_{i} (x) \Big |_{x = h(w ; i)}  \right \|^2 = \sum_{j=1}^c \left( \frac{\partial \phi_i ( x )}{\partial x_j} \Big |_{x = h ( w ; i )}  \right)^2 \leq c. 
\end{align*}
This completes the proof. 
\end{proof}

\section{Proofs of Lemmas and Corollary~\ref{cor_slce_phi_convex}}




\subsection*{Proof of Lemma~\ref{lem_approx_h}}

\begin{proof}

Since $h(\cdot;i)$ are twice continuously differentiable for all $i \in [n]$, we have the following Taylor approximation for each component outputs $h_j(\cdot;i)$ where $ j\in [c]$ and $ i \in [n]$: 
\begin{align*}
    h_j(w^{(t+1)};i) &= h_j(w^{(t)} - \eta^{(t)} v^{(t)} ;i) \\
    &= h_j(w^{(t)};i) - \textbf{J}_{w} h_j ( w ; i )|_{w = w^{(t)}}  \eta^{(t)} v^{(t)} + \frac{1}{2}  (\eta^{(t)} v^{(t)})^\top M_{i,j}(\tilde{w}^{(t)}) (\eta^{(t)}  v^{(t)}), \tagthis \label{eq_approx_h}
\end{align*}
where $M_{i,j}(\tilde{w}^{(t)})$ is the Hessian matrices of 
$h_j(\cdot; i)$
at $\tilde{w}^{(t)}$ and $\tilde{w}^{(t)} = \alpha w^{(t)} + (1-\alpha)w^{(t+1)} $ for some $\alpha \in [0,1]$. 
This leads to our desired statement: 
\begin{align*}
    h(w^{(t+1)};i) = h(w^{(t)} - \eta^{(t)} v^{(t)} ;i) = h(w^{(t)};i) - \eta^{(t)} H_i^{(t)} v^{(t)} + \epsilon_i^{(t)}, 
\end{align*}
where 
\begin{align*}
    \epsilon_{i,j}^{(t)} &= \frac{1}{2} (\eta^{(t)} v^{(t)})^\top M_{i,j}(\tilde{w}^{(t)}) (\eta^{(t)} v^{(t)}), \ j \in [c], 
\end{align*}
Hence we get the final bound:
\begin{align*}
    |\epsilon_{i,j}^{(t)}| &\leq\frac{1}{2} \left| (\eta^{(t)} v^{(t)})^\top M_{i,j}(\tilde{w}^{(t)}) (\eta^{(t)} v^{(t)}) \right|\\
    &\leq\frac{1}{2} (\eta^{(t)})^2 \| v^{(t)} \|^2 \cdot \| M_{i,j}(\tilde{w}^{(t)}) \|  \\
    & \overset{\eqref{bounded_hessian}}{\leq} \frac{1}{2}(\eta^{(t)})^2 \| v^{(t)} \|^2 G ,
    \ j \in [c].
\end{align*}
\end{proof}


\subsection*{Proof of Lemma \ref{lem_stronglyconvex_prob}}
\begin{proof}
From Assumption~\ref{ass_solve_v}, we know that there exists $\hat{v}_{*\varepsilon}^{(t)}$ so that
\begin{align*}
    \frac{1}{2} \frac{1}{n} \sum_{i=1}^{n} \| \eta^{(t)} H_i^{(t)} \hat{v}_{*\varepsilon}^{(t)} - \alpha_i^{(t)} \nabla_z \phi_i (h(w^{(t)};i))  \|^2 \leq \varepsilon^2, 
\end{align*}
and $\| \hat{v}_{*\varepsilon}^{(t)} \|^2 \leq V$, for some $V > 0$. Hence, 
\begin{align*}
\frac{1}{2} \frac{1}{n} \sum_{i=1}^{n} \| \eta^{(t)} H_i^{(t)} \hat{v}_{*\varepsilon}^{(t)} - \alpha_i^{(t)} \nabla_z \phi_i (h(w^{(t)};i))  \|^2 + \frac{\varepsilon^2}{2} \| \hat{v}_{*\varepsilon}^{(t)} \|^2 \leq \varepsilon^2 + \frac{\varepsilon^2}{2} V = ( 1 + \frac{V}{2} ) \varepsilon^2. 
\end{align*}

Since $v_{*\ \text{reg}}^{(t)}$ is the optimal solution of the problem in \eqref{opt_prob_L2} for $0 \leq t < T$, we have
\begin{align*}
    \frac{1}{2} \frac{1}{n} \sum_{i=1}^{n} \| \eta^{(t)} H_i^{(t)} v_{*\ \text{reg}}^{(t)} - \alpha_i^{(t)} \nabla_z \phi_i (h(w^{(t)};i))  \|^2 + \frac{\varepsilon^2}{2} \| v_{*\ \text{reg}}^{(t)} \|^2 \leq ( 1 + \frac{V}{2} ) \varepsilon^2.  
\end{align*}
Therefore, we have \eqref{eq_lem_stronglyconvex_prob} and $\| v_{*\ \text{reg}}^{(t)} \|^2 \leq 2 + V$ for $0 \leq t < T$. 
\end{proof}


\subsection*{Proof of Lemma~\ref{lem_squared_convex_smooth}}


\begin{proof}
1. We want to show that for any $\alpha \in [0, 1]$ 
\begin{align*}
    \phi(\alpha z_1 + (1 - \alpha) z_2) \leq \alpha \phi(z_1) + (1 - \alpha) \phi(z_2), \ \forall z_1, z_2 \in \mathbb{R}^c, \tagthis \label{eq_convex_01}
\end{align*}
in order to have the convexity of $\phi$ with respect to $z$ (see \citep{nesterov2004}). 

For any $\alpha \in [0, 1]$, we have for $\forall z_1, z_2 \in \mathbb{R}^c$, 
\begin{align*}
    & \alpha \| z_1 - b \|^2 + (1 - \alpha) \| z_2 - b \|^2 - \| \alpha (z_1 - b) + (1 - \alpha) (z_2 - b) \|^2 \\
    & = \alpha \| z_1 - b \|^2 + (1 - \alpha) \| z_2 - b \|^2 - \alpha^2 \| z_1 - b \|^2 - (1 - \alpha)^2 \| z_2 - b \|^2 \\ & \qquad - 2 \alpha (1 - \alpha) \langle z_1 - b , z_2 - b \rangle \\
    & \geq \alpha (1 - \alpha) \| z_1 - b \|^2 + (1 - \alpha) \alpha \|z_2 - b \|^2 - 2 \alpha ( 1 - \alpha) \| z_1 - b \| \cdot \| z_2 - b \| \\
    & = \alpha ( 1 - \alpha) \left(  \|z_1 - b\| - \| z_2 - b \|  \right)^2 \geq 0, 
\end{align*}
where the first inequality follows according to Cauchy-Schwarz inequality $\langle a , b \rangle \leq \| a \| \cdot \| b \|$. Hence, 
\begin{align*}
    \frac{1}{2} \| \alpha z_1 + (1 - \alpha) z_2 - b \|^2 \leq \frac{\alpha}{2} \| z_1 - b \|^2 + \frac{(1 - \alpha)}{2} \| z_2 - b \|^2. 
\end{align*}

Therefore, \eqref{eq_convex_01} implies the convexity of $\phi$ with respect to $z$.

2. We want to show that $\exists L_{\phi} > 0$ such that
\begin{align*}
    \| \nabla \phi(z_1) - \nabla \phi(z_2) \| \leq L_{\phi} \| z_1 - z_2 \|, \ \forall z_1, z_2 \in \mathbb{R}^c. \tagthis \label{eq_smooth_01}
\end{align*}

Notice that $\nabla \phi(z) = z - b$, then clearly $\forall z_1, z_2 \in \mathbb{R}^c$,
\begin{align*}
    \| \nabla \phi(z_1) - \nabla \phi(z_2) \| = \| z_1 - z_2 \|. 
\end{align*}

Therefore, \eqref{eq_smooth_01} implies the $L_{\phi}$-smoothness of $\phi$ with respect to $z$ with $L_{\phi} = 1$. 
\end{proof}

\subsection*{Proof of Lemma~\ref{lem_softmax_cross_entropy_convex_smooth}}


\begin{proof}
1. For $\forall z_1, z_2 \in \mathbb{R}^c$ and $1 \leq k \leq c$, denote $u_{k,1} = \exp ( w_k^\top z_1)$ and $\ u_{k,2} = \exp ( w_k^\top z_2)$ and using Holder inequality
\begin{align*}
    \sum_{k=1}^{c} a_k \cdot b_k \leq \left(  \sum_{k=1}^{c} | a_k |^p  \right)^{\frac{1}{p}} \left(  \sum_{k=1}^{c} | b_k |^q  \right)^{\frac{1}{q}}, \ \text{where} \ \frac{1}{p} + \frac{1}{q} = 1, \tagthis \label{holder_inequality}
\end{align*} 
we have
\begin{align*}
    \phi(\alpha z_1 + (1 - \alpha) z_2) &= \log \left[ \sum_{k=1}^{c} \exp ( w_k^\top ( \alpha z_1 + (1 - \alpha) z_2 ) )  \right] = \log \left[ \sum_{k=1}^{c} u_{k,1}^{\alpha} \cdot u_{k,2}^{(1-\alpha)}  \right] \\
    & \overset{\eqref{holder_inequality}}{\leq} \log \left[ \left( \sum_{k=1}^{c} u_{k,1}^{\alpha \cdot \frac{1}{\alpha}} \right)^{\alpha} \left( \sum_{k=1}^{c} u_{k,2}^{(1-\alpha) \cdot \frac{1}{(1-\alpha)}} \right)^{1-\alpha}  \right] \\
    &= \alpha \log \left[ \sum_{k=1}^{c} \exp ( w_k^\top z_1 )  \right] + (1 - \alpha) \log \left[ \sum_{k=1}^{c} \exp ( w_k^\top z_2 )  \right] \\
    & = \alpha \phi(z_1) + (1 - \alpha) \phi(z_2), 
\end{align*}
where the first inequality since  $log(x)$ is an increasing function for $\forall x > 0$ and $\exp(v) > 0$ for $\forall v \in \mathbb{R}$. Therefore, \eqref{eq_convex_01} implies the convexity of $\phi$ with respect to $z$. 

2. Note that $\| \nabla^2 \phi(z) \| \leq L_{\phi}$ if and only if $\phi(z)$ is $L_{\phi}$-smooth (see \citep{nesterov2004}). First, we compute gradient of $\phi(z)$: 
\begin{itemize}
    \item For $i \neq a$:
\begin{align*}
    \frac {\partial \phi(z)} {\partial z_i} 
    = \frac{\exp(z_i - z_a)}{\sum_{k=1}^{c} \exp ( z_k - z_a )}. 
\end{align*}
\item For $i = a$:
\begin{align*}
    \frac {\partial \phi(z)} {\partial z_i} 
    & = \frac{ - \sum_{k \neq a} \exp(z_k - z_a)}{\sum_{k=1}^{c} \exp ( z_k - z_a )}
    = \frac{ - \sum_{k = 1}^c \exp(z_k - z_a) + 1}{\sum_{k=1}^{c} \exp ( z_k - z_a )} \\
    &= -1 + \frac{1}{\sum_{k=1}^{c} \exp ( z_k - z_a )} = -1 + \frac{\exp(z_i - z_a)}{\sum_{k=1}^{c} \exp ( z_k - z_a )}. 
\end{align*}
\end{itemize}

We then calculate $\frac{\partial^2 \phi(z)}{\partial z_j \partial z_i} = \frac{\partial}{\partial z_j} \left( \frac{\partial \phi(z)}{\partial z_i} \right)$
\begin{itemize}
\item For $i = j$:
\begin{align*}
    \frac{\partial^2 \phi(z)}{\partial z_j \partial z_i}
    &= \frac{\exp(z_i - z_a) [\sum_{k=1}^{c} \exp ( z_k - z_a )] - \exp(z_i - z_a) \exp(z_i - z_a)}
    {[\sum_{k=1}^{c} \exp ( z_k - z_a )]^2}\\
    &= \frac{\exp(z_i - z_a) [\sum_{k=1}^{c} \exp ( z_k - z_a ) - \exp(z_i - z_a)]}
    {[\sum_{k=1}^{c} \exp ( z_k - z_a )]^2}. 
\end{align*}
\item For $i \neq j$:
\begin{align*}
    \frac{\partial^2 \phi(z)}{\partial z_j \partial z_i}
    &= \frac{ - \exp(z_j - z_a) \exp(z_i - z_a)}
    {[\sum_{k=1}^{c} \exp ( z_k - z_a )]^2}. 
\end{align*}
\end{itemize}

Denote that $y_i = \exp(z_i - z_a) \geq 0$, $i \in [c]$, we have:
\begin{itemize}
\item For $i = j$:
\begin{align*}
    \left| \frac{\partial^2 \phi(z)}{\partial z_j \partial z_i} \right| 
    &= \left|\frac{y_i (\sum_{k=1}^{c} y_k - y_i)}
    {(\sum_{k=1}^{c} y_k)^2}  \right|. 
\end{align*}
\item For $i \neq j$:
\begin{align*}
    \left|\frac{\partial^2 \phi(z)}{\partial z_j \partial z_i}\right|
    &= \frac{|y_i y_j|} {(\sum_{k=1}^{c} y_k)^2}. 
\end{align*}
\end{itemize}

Recall that for matrix $A = (a_{ij}) \in \R ^{c \times c}$: $\|A\|^2 \leq \|A\|_F^2 = \sum_{i=1}^c \sum_{j=1}^c |a_{ij}|^2$. We have: 
\begin{align*}
    \sum_{j=1}^c \left|\frac{\partial^2 \phi(z)}{\partial z_j \partial z_i}\right|^2 
    &\leq \frac{1} {(\sum_{k=1}^{c} y_k)^4}\left[ y_i^2 (\sum_{k=1}^{c} y_k - y_i)^2 + \sum_{j\neq i} (y_i y_j)^2 \right] \\
    &= \frac{1} {(\sum_{k=1}^{c} y_k)^4}\left[ y_i^2 (\sum_{k=1}^{c} y_k)^2 - 2 y_i^2 \sum_{k=1}^{c} y_k.y_i + y_i^4 + \sum_{j\neq i} (y_i y_j)^2 \right] \\
    &= \frac{1} {(\sum_{k=1}^{c} y_k)^4}\left[ y_i^2 (\sum_{k=1}^{c} y_k)^2 -2 y_i^3 \sum_{k=1}^{c} y_k + y_i^2\sum_{k=1}^c y_k^2 \right]
\end{align*}
Therefore, 
\begin{align*}
    \| \nabla^2 \phi(z)\|^2 & \leq \sum_{i=1}^c \sum_{j=1}^c \left|\frac{\partial^2 \phi(z)}{\partial z_j \partial z_i}\right|^2 \\
    & \leq \frac{1} {(\sum_{k=1}^{c} y_k)^4}\left[ (\sum_{i=1}^c y_i^2) (\sum_{k=1}^{c} y_k)^2 -2 (\sum_{i=1}^c y_i^3) (\sum_{k=1}^{c} y_k) + (\sum_{i=1}^c y_i^2)(\sum_{k=1}^c y_k^2) \right] \\
    & \leq \frac{(\sum_{i=1}^c y_i^2) (\sum_{k=1}^{c} y_k)^2} {(\sum_{k=1}^{c} y_k)^4} \leq \frac{(\sum_{k=1}^c y_k)^4}{(\sum_{k=1}^{c} y_k)^4} = 1, 
\end{align*}
where the last inequality holds since
\begin{align*}
    (\sum_{i=1}^c y_i^2)(\sum_{k=1}^c y_k^2) \leq (\sum_{i=1}^c y_i^3) (\sum_{k=1}^{c} y_k) \Leftrightarrow (\sum_{k=1}^c y_k^2) \leq \sqrt{(\sum_{i=1}^c y_i^3) (\sum_{k=1}^{c} y_k)}, 
\end{align*}
which follows by the application of Holder inequality \eqref{holder_inequality} with $p = 2$, $q = 2$, $a_k = y_k^{3/2}$, and $b_k = y_k^{1/2}$ (Note that $y_k \geq 0$, $k \in [c]$). Hence, $\| \nabla^2 \phi(z) \| \leq L_{\phi}$ with $L_{\phi} = 1$ which is equivalent to $L_{\phi}$-smoothness of $\phi$.
\end{proof}

\subsection*{Proof of Lemma~\ref{lem_approx_h.2}}

\begin{proof}

Since $k(\cdot;i)$ are twice continuously differentiable for all $i \in [n]$, we have the following Taylor approximation for each component outputs $k_j(\cdot;i)$ where $ j\in [c]$ and $ i \in [n]$: 
\begin{align*}
    k_j(w^{(t+1)};i) &= k_j(w^{(t)} - \eta^{(t)} v^{(t)} ;i) \\
    &= k_j(w^{(t)};i) - \textbf{J}_{w} k_j ( w ; i )|_{w = w^{(t)}}  \eta^{(t)} v^{(t)} + \frac{1}{2}  (\eta^{(t)} v^{(t)})^\top M_{i,j}(\tilde{w}^{(t)}) (\eta^{(t)}  v^{(t)}), \tagthis \label{eq_approx_k}
\end{align*}
where $M_{i,j}(\tilde{w}^{(t)})$ is the Hessian matrices of 
$k_j(\cdot; i)$
at $\tilde{w}^{(t)}$ and $\tilde{w}^{(t)} = \alpha w^{(t)} + (1-\alpha)w^{(t+1)} $ for some $\alpha \in [0,1]$. 

Shifting this back to the original function $h_j(\cdot;i)$ we have: 
\begin{align*}
    h_j(w^{(t+1)};i)
    &= k_j(w^{(t+1)};i) + (h_j(w^{(t+1)};i) - k_j(w^{(t+1)};i))  \\
    &\overset{\eqref{eq_approx_k}}{=} k_j(w^{(t)};i) - \textbf{J}_{w} k_j ( w ; i )|_{w = w^{(t)}}  \eta^{(t)} v^{(t)} + \frac{1}{2}  (\eta^{(t)} v^{(t)})^\top M_{i,j}(\tilde{w}^{(t)}) (\eta^{(t)} v^{(t)}) \\
    &\quad + (h_j(w^{(t+1)};i) - k_j(w^{(t+1)};i)) ,\\
    &= h_j(w^{(t)};i) - \textbf{J}_{w} k_j ( w ; i )|_{w = w^{(t)}}  \eta^{(t)} v^{(t)} + \frac{1}{2}  (\eta^{(t)} v^{(t)})^\top M_{i,j}(\tilde{w}^{(t)}) (\eta^{(t)} v^{(t)}) \\
    &\quad + (h_j(w^{(t+1)};i) - k_j(w^{(t+1)};i)) + (k_j(w^{(t)};i) - h_j(w^{(t)};i)) ,
\end{align*}
which leads to our desired statement: 
\begin{align*}
    h(w^{(t+1)};i) = h(w^{(t)} - \eta^{(t)} v^{(t)} ;i) = h(w^{(t)};i) - \eta^{(t)} H_i^{(t)} v^{(t)} + \epsilon_i^{(t)}, 
\end{align*}
where 
\begin{align*}
    \epsilon_{i,j}^{(t)} &= \frac{1}{2} (\eta^{(t)} v^{(t)})^\top M_{i,j}(\tilde{w}^{(t)}) (\eta^{(t)} v^{(t)}) \\
    & \quad + (h_j(w^{(t+1)};i) - k_j(w^{(t+1)};i)) + (k_j(w^{(t)};i) - h_j(w^{(t)};i)), \ j \in [c], 
\end{align*}
Hence we get the final bound:
\begin{align*}
    |\epsilon_{i,j}^{(t)}| &\leq\frac{1}{2} \left| (\eta^{(t)} v^{(t)})^\top M_{i,j}(\tilde{w}^{(t)}) (\eta^{(t)} v^{(t)}) \right|\\
    &\quad +|h_j(w^{(t+1)};i) - k_j(w^{(t+1)};i)| + |k_j(w^{(t)};i) - h_j(w^{(t)};i)|\\
    &\overset{\eqref{eq_approx_eps}}{\leq}\frac{1}{2} \left| (\eta^{(t)} v^{(t)})^\top M_{i,j}(\tilde{w}^{(t)}) (\eta^{(t)} v^{(t)}) \right| + 2 \varepsilon, \\
    &\leq\frac{1}{2} (\eta^{(t)})^2 \| v^{(t)} \|^2 \cdot \| M_{i,j}(\tilde{w}^{(t)}) \| + 2\varepsilon \\
    & \overset{\eqref{bounded_hessian}}{\leq} \frac{1}{2}(\eta^{(t)})^2 \| v^{(t)} \|^2 G + 2\varepsilon,
    \ j \in [c].
\end{align*}
\end{proof}

\subsection*{Proof of Corollary~\ref{cor_slce_phi_convex}}

\begin{proof}
The proof of this corollary follows directly by the applications of Lemmas~\ref{lem_squared_convex_smooth} and \ref{lem_softmax_cross_entropy_convex_smooth}. 
\end{proof}


\section{Technical Proofs for Theorem \ref{thm_main_result_closed_form}}\label{sec_proofs_closed_form}

\begin{lemma}\label{lem_bound_epsilon_closed_form}
Suppose that Assumption \ref{ass_approx_h} holds for $G > 0$ and Assumption \ref{ass_solve_v} holds for $V > 0$, and $v^{(t)} = v_{*\ \text{reg}}^{(t)}$. Consider $\eta^{(t)} = D\sqrt{\varepsilon}$ for some $D > 0$ and $\varepsilon > 0$. For $i \in [n]$ and $0 \leq t < T$, we have
\begin{align*}
    \| \epsilon_i^{(t)} \|^2 \leq \frac{1}{4} c (4+(V+2) G D^2)^2 \varepsilon^2. \tagthis \label{eq_lem_bound_epsilon_closed_form}
\end{align*}
\end{lemma}

\begin{proof}
From \eqref{eq_epsilon}, for $i \in [n]$, $j \in [c]$, and for $0 \leq t < T$, by Lemma~\ref{lem_approx_h} and Lemma \ref{lem_approx_h.2} we have
\begin{align*}
    | \epsilon_{i,j}^{(t)} | & \leq  \frac{1}{2}(\eta^{(t)})^2 \| v^{(t)} \|^2 G + 2\varepsilon\leq \frac{1}{2} (V+2) G D^2 \varepsilon + 2\varepsilon = \frac{1}{2}\varepsilon (4+(V+2) G D^2),  
\end{align*}
where the last inequality follows by the fact $ \| v^{(t)} \|^2 = \| v_{*\ \text{reg}}^{(t)} \|^2 \leq 2 + V$ of Lemma~\ref{lem_stronglyconvex_prob} and $\eta^{(t)} = D\sqrt{\varepsilon}$. Hence, 
\begin{align*}
    \| \epsilon_i^{(t)} \|^2 = \sum_{j=1}^c | \epsilon_{i,j}^{(t)} |^2 \leq \frac{1}{4} c (4+(V+2) G D^2)^2 \varepsilon^2. 
\end{align*}
\end{proof}

\begin{lemma}\label{lem_main_result_closed_form}

Let $w^{(t)}$ be generated by Algorithm \ref{alg_closed_form} where we use the closed form solution for the search direction.
We execute Algorithm \ref{alg_closed_form} for $T = \frac{\beta}{\varepsilon}$ outer loops for some constant $\beta>0$.
We assume Assumption \ref{ass_phi} holds.
Suppose that Assumption \ref{ass_approx_h} holds for $G > 0$ and Assumption \ref{ass_solve_v} holds for $V > 0$.
We set the step size equal to $\eta^{(t)} = D\sqrt{\varepsilon}$ for some $D > 0$ and choose a learning rate  
$\alpha_i^{(t)} \leq \frac{\alpha}{L_{\phi}}$, for some $\alpha \in (0, \frac{1}{3})$. For $i \in [n]$ and $0 \leq t < T$, we have
\begin{align*}
    \| h(w^{(t+1)};i) - h_i^* \|^2 
   & \leq (1 + \varepsilon) \| h(w^{(t)};i) - h_i^* \|^2 - 2 ( 1 - 3 \alpha) \alpha_i^{(t)} [ \phi_i (h(w^{(t)};i)  ) - \phi_i (h_i^* ) ]\\
    & \quad
    + \frac{(3\varepsilon+2)}{4} c (4+(V+2) G D^2)^2 \cdot \varepsilon\\
    & \quad + \frac{3\varepsilon+2}{\varepsilon} \| \eta^{(t)} H_i^{(t)} v_{*\ \text{reg}}^{(t)} - \alpha_i^{(t)} \nabla_z \phi_i (h(w^{(t)};i)) \|^2  \tagthis \label{eq_lem_main_result_closed_form}
\end{align*}

\end{lemma}

\begin{proof}
Note that we have the optimal solution $v_{*\ \text{reg}}^{(t)}$ for the optimization problem \eqref{opt_prob_L2} for $0 \leq t < T$. From \eqref{eq_002.3}, we have, for $i \in [n]$,
\begin{align*}
    h(w^{(t+1)};i) & =   h(w^{(t)} - \eta^{(t)} v_{*\ \text{reg}}^{(t)} ;i) \\
    & = h(w^{(t)};i) - \eta^{(t)} H_i^{(t)} v_{*\ \text{reg}}^{(t)} + \epsilon_i^{(t)} \\
    &= h(w^{(t)};i) - \alpha_i^{(t)} \nabla_z \phi_i (h(w^{(t)};i)) + \epsilon_i^{(t)} - [ \eta^{(t)} H_i^{(t)} v_{*\ \text{reg}}^{(t)} - \alpha_i^{(t)} \nabla_z \phi_i (h(w^{(t)};i)) ]. 
\end{align*}

Hence, we have
\allowdisplaybreaks
\begin{align*}
    &\qquad \| h(w^{(t+1)};i) - h_i^* \|^2\\
    &= \| h(w^{(t)};i) - h_i^*  - \alpha_i^{(t)} \nabla_z \phi_i (h(w^{(t)};i))  + \epsilon_i^{(t)}  - [ \eta^{(t)} H_i^{(t)} v_{*\ \text{reg}}^{(t)} - \alpha_i^{(t)} \nabla_z \phi_i (h(w^{(t)};i)) ] \|^2 \\
    &= \| h(w^{(t)};i) - h_i^* \|^2   + ( \alpha_i^{(t)} )^2 \| \nabla_z \phi_i (h(w^{(t)};i)) \|^2  \\ & \quad + \| \epsilon_i^{(t)} \|^2+ \| \eta^{(t)} H_i^{(t)} v_{*\ \text{reg}}^{(t)} - \alpha_i^{(t)} \nabla_z \phi_i (h(w^{(t)};i)) \|^2 \\
    & \quad- 2 \cdot \langle  h(w^{(t)};i) - h_i^* ,\alpha_i^{(t)}  \nabla_z \phi_i (h(w^{(t)};i)) \rangle \\
    & \quad+ 2 \cdot \langle h(w^{(t)};i) - h_i^* , \epsilon_i^{(t)} \rangle  \\
    & \quad- 2 \cdot\langle h(w^{(t)};i) - h_i^* , \eta^{(t)} H_i^{(t)} v_{*\ \text{reg}}^{(t)} - \alpha_i^{(t)} \nabla_z \phi_i (h(w^{(t)};i))  \rangle \\
    & \quad- 2 \cdot \langle \alpha_i^{(t)} \nabla_z \phi_i (h(w^{(t)};i)) , \epsilon_i^{(t)} \rangle \\
    & \quad + 2 \cdot \langle \alpha_i^{(t)} \nabla_z \phi_i (h(w^{(t)};i)), \eta^{(t)} H_i^{(t)} v_{*\ \text{reg}}^{(t)} - \alpha_i^{(t)} \nabla_z \phi_i (h(w^{(t)};i))  \rangle \\
    & \quad - 2 \cdot\langle \epsilon_i^{(t)}, \eta^{(t)} H_i^{(t)} v_{*\ \text{reg}}^{(t)} - \alpha_i^{(t)} \nabla_z \phi_i (h(w^{(t)};i))  \rangle,
\end{align*}
where we expand the square term.
Now applying Young's inequalities: $2| \langle u , v \rangle |  \leq \frac{\| u \|^2}{\varepsilon/2} + (\varepsilon/2) \| v \|^2$ for $\varepsilon > 0$ and $2| \langle u , v \rangle | \leq \| u \|^2 + \| v \|^2$ we have:
\begin{align*}
    &\qquad \| h(w^{(t+1)};i) - h_i^* \|^2\\
    &= \| h(w^{(t)};i) - h_i^* \|^2 
    + ( \alpha_i^{(t)} )^2 \| \nabla_z \phi_i (h(w^{(t)};i)) \|^2  \\ & \quad + \| \epsilon_i^{(t)} \|^2+ \| \eta^{(t)} H_i^{(t)} v_{*\ \text{reg}}^{(t)} - \alpha_i^{(t)} \nabla_z \phi_i (h(w^{(t)};i)) \|^2 \\
    %
    & \quad - 2 \alpha_i^{(t)} \langle  h(w^{(t)};i) - h_i^* , \nabla_z \phi_i (h(w^{(t)};i)) \rangle\\ 
    & \quad
    + \frac{\varepsilon}{2} \| h(w^{(t)};i) - h_i^* \|^2 +  \frac{2}{\varepsilon} \| \epsilon_i^{(t)} \|^2 \\
    & \quad + \frac{\varepsilon}{2} \| h(w^{(t)};i) - h_i^* \|^2 + \frac{2}{\varepsilon} \| \eta^{(t)} H_i^{(t)} v_{*\ \text{reg}}^{(t)} - \alpha_i^{(t)} \nabla_z \phi_i (h(w^{(t)};i))  \|^2\\
    & \quad 
    + 2 ( \alpha_i^{(t)} )^2 \| \nabla_z \phi_i (h(w^{(t)};i)) \|^2 + 2 \| \epsilon_i^{(t)} \|^2 \\
    & \quad + 2 \| \eta^{(t)} H_i^{(t)} v_{*\ \text{reg}}^{(t)} - \alpha_i^{(t)} \nabla_z \phi_i (h(w^{(t)};i)) \|^2 \\
    & \overset{\eqref{eq_phi_convex}}{\leq} (1 + \varepsilon) \| h(w^{(t)};i) - h_i^* \|^2  + 3 ( \alpha_i^{(t)} )^2 \| \nabla_z \phi_i (h(w^{(t)};i)) \|^2\\
    & \quad 
    + \left( 3+ \frac{2}{\varepsilon} \right) \| \epsilon_i^{(t)} \|^2   + \left( 3+ \frac{2}{\varepsilon} \right) \| \eta^{(t)} H_i^{(t)} v_{*\ \text{reg}}^{(t)} - \alpha_i^{(t)} \nabla_z \phi_i (h(w^{(t)};i)) \|^2 \\ 
    & \quad - 2 \alpha_i^{(t)} [ \phi_i (h(w^{(t)};i) ) - \phi_i (h_i^* ) ]. 
\end{align*}

Note that from \eqref{eq_phi_convex_smooth} we get that $\| \nabla_z \phi_i (h(w^{(t)};i)) \|^2 \leq 2 L_{\phi} [ \phi_i (h(w^{(t)};i)) - \phi_i (h_i^* ) ]$. Applying this and using the fact that $\alpha_i^{(t)} \leq \frac{\alpha}{L_{\phi}}$, for some $\alpha \in (0, \frac{1}{3})$, we are able to derive:
\begin{align*}
    &\qquad \| h(w^{(t+1)};i) - h_i^* \|^2 \\
    & \leq (1 + \varepsilon) \| h(w^{(t)};i) - h_i^* \|^2 - 2 ( 1 - 3 \alpha) \alpha_i^{(t)} [ \phi_i (h(w^{(t)};i)  ) - \phi_i (h_i^* ) ] \\
    & \quad 
    + \frac{3\varepsilon+2}{\varepsilon}\| \epsilon_i^{(t)} \|^2  + \frac{3\varepsilon+2}{\varepsilon} \| \eta^{(t)} H_i^{(t)} v_{*\ \text{reg}}^{(t)} - \alpha_i^{(t)} \nabla_z \phi_i (h(w^{(t)};i)) \|^2 \\ 
    %
    %
    & \leq (1 + \varepsilon) \| h(w^{(t)};i) - h_i^* \|^2 - 2 ( 1 - 3 \alpha) \alpha_i^{(t)} [ \phi_i (h(w^{(t)};i)  ) - \phi_i (h_i^* ) ]\\
    & \quad
    + \frac{3\varepsilon+2}{\varepsilon} \frac{1}{4} c (4+(V+2) G D^2)^2 \varepsilon^2  + \frac{3\varepsilon+2}{\varepsilon} \| \eta^{(t)} H_i^{(t)} v_{*\ \text{reg}}^{(t)} - \alpha_i^{(t)} \nabla_z \phi_i (h(w^{(t)};i)) \|^2 
\end{align*}

where the last inequality follows by Lemma~\ref{lem_bound_epsilon_closed_form}. 
\end{proof}

\begin{lemma}\label{lem_main_result_closed_form_02}
Let $w^{(t)}$ be generated by Algorithm \ref{alg_closed_form} where we use the closed form solution for the search direction.
We execute Algorithm \ref{alg_closed_form} for $T = \frac{\beta}{\varepsilon}$ outer loops for some constant $\beta>0$.
We assume Assumption \ref{ass_phi} holds.
Suppose that Assumption \ref{ass_approx_h} holds for $G > 0$ and Assumption \ref{ass_solve_v} holds for $V > 0$.
We set the step size equal to $\eta^{(t)} = D\sqrt{\varepsilon}$ for some $D > 0$ and choose a learning rate 
$ \alpha_i^{(t)} = ( 1 + \varepsilon ) \alpha_i^{(t-1)} = ( 1 + \varepsilon )^t \alpha_i^{(0)}$.  Based on $\beta$, we define 
$\alpha_i^{(0)} = \frac{\alpha}{e^{\beta} L_{\phi}}$
with $\alpha \in (0, \frac{1}{3})$.
We have
\allowdisplaybreaks
\begin{align*}
    \frac{1}{T} \sum_{t=0}^{T-1} \frac{1}{n} \sum_{i=1}^n [ f (w^{(t)}; i) - \phi_i (h_i^* )  ] &\leq \frac{e^{\beta} L_{\phi}( 1 + \varepsilon )}{2 (1-3\alpha) \alpha \beta}  \cdot  \frac{1}{n} \sum_{i=1}^n  \| h(w^{(0)};i) - h_i^* \|^2 \cdot \varepsilon\\
    & \quad + \frac{e^{\beta} L_{\phi}}{8\alpha(1-3\alpha)}  (3\varepsilon+2) \left[ c (4+(V+2) G D^2)^2  +8 + 4V\right]\cdot \varepsilon . \tagthis \label{eq_lem_main_result_cf_02}
\end{align*}
\end{lemma}

\begin{proof}
Rearranging the terms in Lemma~\ref{lem_main_result_closed_form},
we have
\allowdisplaybreaks
\begin{align*}
     \phi_i (h(w^{(t)};i) ) - \phi_i (h_i^* ) & \leq \frac{1}{2 (1-3\alpha)}\left( \frac{(1 + \varepsilon)}{\alpha_i^{(t)}} \| h(w^{(t)};i) - h_i^* \|^2 - \frac{1}{\alpha_i^{(t)}} \| h(w^{(t+1)};i) - h_i^* \|^2 \right) \\
    & \quad + \frac{1}{8(1-3\alpha)} \cdot \frac{1}{\alpha_i^{(t)}} \cdot \varepsilon(3\varepsilon+2)  c (4+(V+2) G D^2)^2  \\
    & \quad
    + \frac{1}{2(1-3\alpha)} \cdot \frac{1}{\alpha_i^{(t)}} \cdot\frac{3\varepsilon+2}{\varepsilon} \| \eta^{(t)} H_i^{(t)} v_{*\ \text{reg}}^{(t)} - \alpha_i^{(t)} \nabla_z \phi_i (h(w^{(t)};i)) \|^2   \\
    & \leq \frac{1}{2 (1-3\alpha)}\left( \frac{(1 + \varepsilon)}{\alpha_i^{(t)}}  \| h(w^{(t)};i) - h_i^* \|^2 - \frac{(1 + \varepsilon)}{\alpha_i^{(t+1)}} \| h(w^{(t+1)};i) - h_i^* \|^2 \right) \\
    & \quad + \frac{e^{\beta} L_{\phi}}{8\alpha(1-3\alpha)}  \cdot \varepsilon(3\varepsilon+2)  c (4+(V+2) G D^2)^2  \\
    & \quad
    + \frac{e^{\beta} L_{\phi}}{2\alpha(1-3\alpha)}  \cdot\frac{3\varepsilon+2}{\varepsilon} \| \eta^{(t)} H_i^{(t)} v_{*\ \text{reg}}^{(t)} - \alpha_i^{(t)} \nabla_z \phi_i (h(w^{(t)};i)) \|^2 .
\end{align*}
The last inequality follows because the learning rate satisfies $\alpha_i^{(0)} = \frac{\alpha}{e^{\beta} L_{\phi}} \leq \frac{\alpha}{L_{\phi}}$ and for $t = 1,\dots,T = \frac{\beta}{\varepsilon}$ for some $\beta > 0$
\begin{align*}
    \alpha_i^{(t)} = ( 1 + \varepsilon ) \alpha_i^{(t-1)} = ( 1 + \varepsilon )^t \alpha_i^{(0)} \leq ( 1 + \varepsilon )^T \alpha_i^{(0)} = ( 1 + \varepsilon )^{\beta/\varepsilon} \frac{\alpha}{e^{\beta} L_{\phi}} \leq \frac{\alpha}{L_{\phi}}, 
\end{align*}
since $(1 + x)^{1/x} \leq e$, $x > 0$. Moreover, we have $\frac{1}{\alpha_i^{(t)}} \leq \frac{1}{\alpha_i^{(0)}} = \frac{e^{\beta} L_{\phi}}{\alpha}$, $t = 0,\dots,T-1$. 

Taking the average sum from $t = 0,\dots,T-1$, we have
\begin{align*}
    \frac{1}{T} \sum_{t=0}^{T-1} [\phi_i (h(w^{(t)};i) ) - \phi_i (h_i^* ) ] & \leq \frac{1}{2 (1-3\alpha) T} \cdot \frac{( 1 + \varepsilon ) }{\alpha_i^{(0)}} \| h(w^{(0)};i) - h_i^* \|^2 \\
    & \quad + \frac{e^{\beta} L_{\phi}}{8\alpha(1-3\alpha)}  \cdot \varepsilon(3\varepsilon+2)  c (4+(V+2) G D^2)^2  \\
    & \quad
    + \frac{e^{\beta} L_{\phi}}{2\alpha(1-3\alpha)}  \cdot\frac{3\varepsilon+2}{\varepsilon} \frac{1}{T} \sum_{t=0}^{T-1} \| \eta^{(t)} H_i^{(t)} v_{*\ \text{reg}}^{(t)} - \alpha_i^{(t)} \nabla_z \phi_i (h(w^{(t)};i)) \|^2 \\
    & = \frac{e^{\beta} L_{\phi}( 1 + \varepsilon )}{2 (1-3\alpha) \alpha \beta} \varepsilon \cdot \| h(w^{(0)};i) - h_i^* \|^2 \\
    & \quad + \frac{e^{\beta} L_{\phi}}{8\alpha(1-3\alpha)}  \cdot \varepsilon(3\varepsilon+2) c (4+(V+2) G D^2)^2  \\
    & \quad
    + \frac{e^{\beta} L_{\phi}}{2\alpha(1-3\alpha)}  \cdot\frac{3\varepsilon+2}{\varepsilon} \frac{1}{T} \sum_{t=0}^{T-1} \| \eta^{(t)} H_i^{(t)} v_{*\ \text{reg}}^{(t)} - \alpha_i^{(t)} \nabla_z \phi_i (h(w^{(t)};i)) \|^2.
\end{align*}
Taking the average sum from $i = 1, \dots, n$, we have
\begin{align*}
    & \frac{1}{T} \sum_{t=0}^{T-1} \frac{1}{n} \sum_{i=1}^n [ \phi_i (h(w^{(t)};i) ) - \phi_i (h_i^* ) ]  \\ 
    & \leq \frac{e^{\beta} L_{\phi}( 1 + \varepsilon )}{2 (1-3\alpha) \alpha \beta} \varepsilon \cdot  \frac{1}{n} \sum_{i=1}^n  \| h(w^{(0)};i) - h_i^* \|^2 \\
    & \quad + \frac{e^{\beta} L_{\phi}}{8\alpha(1-3\alpha)}  \cdot \varepsilon(3\varepsilon+2)  c (4+(V+2) G D^2)^2  \\
    & \quad
    + \frac{e^{\beta} L_{\phi}}{2\alpha(1-3\alpha)}  \cdot\frac{3\varepsilon+2}{\varepsilon} \frac{1}{T} \sum_{t=0}^{T-1} \frac{1}{n} \sum_{i=1}^n  \| \eta^{(t)} H_i^{(t)} v_{*\ \text{reg}}^{(t)} - \alpha_i^{(t)} \nabla_z \phi_i (h(w^{(t)};i)) \|^2\\
    & \overset{\eqref{eq_lem_stronglyconvex_prob}}{\leq}  \frac{e^{\beta} L_{\phi}( 1 + \varepsilon )}{2 (1-3\alpha) \alpha \beta} \varepsilon \cdot  \frac{1}{n} \sum_{i=1}^n  \| h(w^{(0)};i) - h_i^* \|^2 \\
    & \quad + \frac{e^{\beta} L_{\phi}}{8\alpha(1-3\alpha)}  \cdot \varepsilon(3\varepsilon+2)  c (4+(V+2) G D^2)^2  \\
    & \quad
    + \frac{e^{\beta} L_{\phi}}{2\alpha(1-3\alpha)}  \cdot\frac{3\varepsilon+2}{\varepsilon} (2+V) \varepsilon^2. \tagthis \label{eq_0002_cf_new}
\end{align*}

Note that
\begin{align*}
     \frac{1}{T} \sum_{t=0}^{T-1} \frac{1}{n} \sum_{i=1}^n [ \phi_i (h(w^{(t)};i) ) - \phi_i (h_i^* ) ]  = \frac{1}{T} \sum_{t=0}^{T-1} \frac{1}{n} \sum_{i=1}^n [ f (w^{(t)}; i) -\phi_i (h_i^* ) ]. \tagthis \label{eq_average_phi_cf_02}
\end{align*}

Therefore, applying \eqref{eq_average_phi_cf_02} to \eqref{eq_0002_cf_new}, we have
\begin{align*}
    \frac{1}{T} \sum_{t=0}^{T-1} \frac{1}{n} \sum_{i=1}^n [ f (w^{(t)}; i) - \phi_i (h_i^* )  ] 
    & \leq  \frac{e^{\beta} L_{\phi}( 1 + \varepsilon )}{2 (1-3\alpha) \alpha \beta}  \cdot  \frac{1}{n} \sum_{i=1}^n  \| h(w^{(0)};i) - h_i^* \|^2 \cdot \varepsilon\\
    & \quad + \frac{e^{\beta} L_{\phi}}{8\alpha(1-3\alpha)}  (3\varepsilon+2) \left[ c (4+(V+2) G D^2)^2  +8 + 4V\right]\cdot \varepsilon . 
\end{align*}
which is our desired result. 
\end{proof}

\subsection*{Proof of Theorem~\ref{thm_main_result_closed_form}}

\begin{proof}
We have 
\begin{align*}
    F_* &= \min_{w \in \mathbb{R}^d} F ( w )  = \min_{w \in \mathbb{R}^d} \left(\frac{1}{n}\sum_{i=1}^n f_i (w) \right)= \frac{1}{n} \min_{w \in \mathbb{R}^d} \left(\sum_{i=1}^n f_i (w) \right) \\
    &\geq \frac{1}{n} \sum_{i=1}^n \min_{w \in \mathbb{R}^d} \left( f_i (w) \right) = \frac{1}{n} \sum_{i=1}^n f_i^* \geq \frac{1}{n} \sum_{i=1}^n \phi_i (h_i^* ). \tagthis \label{eq_loss_great_01}
\end{align*}
Hence $F_* - \frac{1}{n} \sum_{i=1}^n \phi_i (h_i^* ) \geq 0$. Therefore
\allowdisplaybreaks
\begin{align*}
    \frac{1}{T} \sum_{t=0}^{T-1} [ F (w^{(t)}) - F_* ] & = \frac{1}{T} \sum_{t=0}^{T-1} \left( \frac{1}{n} \sum_{i=1}^n [ f (w^{(t)}; i) - \phi_i (h_i^* )  ] - \left[ F_* - \frac{1}{n} \sum_{i=1}^{n} \phi_i (h_i^* )  \right] \right) \\
    & \leq \frac{1}{T} \sum_{t=0}^{T-1} \frac{1}{n} \sum_{i=1}^n [ f (w^{(t)}; i) - \phi_i (h_i^* )  ] \\
    & \overset{\eqref{eq_lem_main_result_cf_02}
    }{\leq}  \frac{e^{\beta} L_{\phi}( 1 + \varepsilon )}{2 (1-3\alpha) \alpha \beta} \cdot  \frac{1}{n} \sum_{i=1}^n  \| h(w^{(0)};i) - h_i^* \|^2 \cdot \varepsilon \\
    & \quad + \frac{e^{\beta} L_{\phi}(3\varepsilon+2)}{8\alpha(1-3\alpha)}  \left[ c (4+(V+2) G D^2)^2  +8 + 4V\right] \cdot \varepsilon. 
\end{align*}
\end{proof}


\section{Technical Proofs for Theorem \ref{thm_main_result_GD_solved.3}}\label{sec_proofs_GD_solved.3}

\begin{lemma}\label{lem_bound_v.3}
For $0 \leq t < T$, suppose that Assumption  \ref{ass_solve_v} holds for $V \geq 0$ and $v^{(t)}$ satisfies \eqref{criteria_eps_sol.3}. Then 
\begin{align*}
    \| v^{(t)} \|^2 \leq 2(\varepsilon^2 + V+2). 
\end{align*}
\end{lemma}

\begin{proof} From $\| v^{(t)} - v_{*\ \text{reg}}^{(t)} \| \leq \varepsilon$.
Using $\| a \|^2 \leq 2\| a - b \|^2+2 \| b \|^2 $, we have
\begin{align*}
    \| v^{(t)} \|^2 \leq 2\| v^{(t)} - v_{*\ \text{reg}}^{(t)} \|^2  + 2 \| v_{*\ \text{reg}}^{(t)}\|^2 \overset{\eqref{criteria_eps_sol.3}}{\leq} 2\varepsilon^2  + 4 + 2 V. 
\end{align*}
where the last inequality follows since $\| v_{*\ \text{reg}}^{(t)} \|^2 \leq 2+ V$ for some $V > 0$ in Lemma~\ref{lem_stronglyconvex_prob}.

\end{proof}

\begin{lemma}\label{lem_bound_epsilon.3}
Suppose that Assumption \ref{ass_approx_h} holds for $G > 0$ and Assumption \ref{ass_solve_v} holds for $V > 0$. Consider $\eta^{(t)} = D\sqrt{\varepsilon}$ for some $D > 0$ and $\varepsilon > 0$. For $i \in [n]$ and $0 \leq t < T$, we have
\begin{align*}
    \| \epsilon_i^{(t)} \|^2 \leq c (2+(V+\varepsilon^2 +2) G D^2)^2\varepsilon^2  .  \tagthis \label{eq_lem_bound_epsilon.2}
\end{align*}
\end{lemma}

\begin{proof}
From \eqref{eq_epsilon}, for $i \in [n]$, $j \in [c]$, and for $0 \leq t < T$, by Lemma~\ref{lem_approx_h} and Lemma \ref{lem_approx_h.2} we have
\begin{align*}
    | \epsilon_{i,j}^{(t)} | & \leq  \frac{1}{2}(\eta^{(t)})^2 \| v^{(t)} \|^2 G + 2\varepsilon \leq \frac{1}{2}2( \varepsilon^2 + V+2) G D^2 \varepsilon + 2\varepsilon = \varepsilon (2+(V+\varepsilon^2 +2) G D^2) ,  
\end{align*}
where the last inequality follows by the application of Lemma~\ref{lem_bound_v.3} and $\eta^{(t)} = D \sqrt {\varepsilon}$. Hence, 
\begin{align*}
    \| \epsilon_i^{(t)} \|^2 = \sum_{j=1}^c | \epsilon_{i,j}^{(t)} |^2 \leq c (2+(V+\varepsilon^2 +2) G D^2)^2\varepsilon^2  . 
\end{align*}
\end{proof}

\begin{lemma}\label{lem_main_result_01_new.3}
Let $w^{(t)}$ be generated by Algorithm \ref{alg_GD_solved.2}
where $v^{(t)}$ satisfies \eqref{criteria_eps_sol.3}.
 We execute Algorithm \ref{alg_GD_solved.2} for $T = \frac{\beta}{\varepsilon}$ outer loops for some constant $\beta>0$.
We assume Assumption \ref{ass_phi} holds.
Suppose that Assumption \ref{ass_approx_h} holds for $G > 0$, Assumption \ref{ass_solve_v} holds for $V > 0$
and Assumption \ref{ass_bounded_hessian.2} holds for $H > 0$. 
We set the step size equal to $\eta^{(t)} = D\sqrt{\varepsilon}$ for some $D > 0$ and choose a learning rate 
$\alpha_i^{(t)} \leq \frac{\alpha}{L_{\phi}}$, for some $\alpha \in (0, \frac{1}{4})$. For $i \in [n]$ and $0 \leq t < T$, we have

\begin{align*}
    \| h(w^{(t+1)};i) - h_i^* \|^2 &\leq  (1 + \varepsilon) \| h(w^{(t)};i) - h_i^* \|^2 - 2 ( 1 - 4 \alpha) \alpha_i^{(t)} [ \phi_i (h(w^{(t)};i)  ) - \phi_i (h_i^* ) ]\\
    & \quad + \varepsilon(4\varepsilon+3)\left[D^2H^2 + c (2+(V+\varepsilon^2 +2) G D^2)^2\right]\\
    & \quad + \frac{4\varepsilon+3}{\varepsilon} \| \eta^{(t)} H_i^{(t)} v_{*\ \text{reg}}^{(t)} - \alpha_i^{(t)} \nabla_z \phi_i (h(w^{(t)};i)) \|^2 \tagthis \label{eq_lem_main_result_01_new.2}
\end{align*}
\end{lemma}

\begin{proof}
Note that $v^{(t)}$ is obtained from the optimization problem \eqref{opt_prob_L2} for $0 \leq t < T$. From \eqref{eq_002.1}, we have, for $i \in [n]$,
\begin{align*}
    h(w^{(t+1)};i) & =   h(w^{(t)} - \eta^{(t)} v^{(t)} ;i) \\
    & = h(w^{(t)};i) - \eta^{(t)} H_i^{(t)} v^{(t)} + \epsilon_i^{(t)} \\
    &= h(w^{(t)};i) - \eta^{(t)} H_i^{(t)} (v^{(t)} - v_{*\ \text{reg}}^{(t)}) - \alpha_i^{(t)} \nabla_z \phi_i (h(w^{(t)};i)) + \epsilon_i^{(t)} \\ & \qquad - [ \eta^{(t)} H_i^{(t)} v_{*\ \text{reg}}^{(t)} - \alpha_i^{(t)} \nabla_z \phi_i (h(w^{(t)};i)) ]. 
\end{align*}

Hence, we have
\allowdisplaybreaks
\begin{align*}
    &\qquad \| h(w^{(t+1)};i) - h_i^* \|^2\\
    &= \| h(w^{(t)};i) - h_i^* - \eta^{(t)} H_i^{(t)} (v^{(t)} - v_{*\ \text{reg}}^{(t)}) - \alpha_i^{(t)} \nabla_z \phi_i (h(w^{(t)};i)) \\&\qquad + \epsilon_i^{(t)}  - [ \eta^{(t)} H_i^{(t)} v_{*\ \text{reg}}^{(t)} - \alpha_i^{(t)} \nabla_z \phi_i (h(w^{(t)};i)) ] \|^2 \\
    &= \| h(w^{(t)};i) - h_i^* \|^2 +  \| \eta^{(t)}  H_i^{(t)} (v^{(t)} - v_{*\ \text{reg}}^{(t)})  \|^2  + ( \alpha_i^{(t)} )^2 \| \nabla_z \phi_i (h(w^{(t)};i)) \|^2  \\ & \quad + \| \epsilon_i^{(t)} \|^2+ \| \eta^{(t)} H_i^{(t)} v_{*\ \text{reg}}^{(t)} - \alpha_i^{(t)} \nabla_z \phi_i (h(w^{(t)};i)) \|^2 \\
    & \quad - 2 \cdot \langle h(w^{(t)};i) - h_i^* , \eta^{(t)} H_i^{(t)} (v^{(t)} - v_{*\ \text{reg}}^{(t)}) \rangle \\
    & \quad- 2 \cdot \langle  h(w^{(t)};i) - h_i^* ,\alpha_i^{(t)}  \nabla_z \phi_i (h(w^{(t)};i)) \rangle \\
    & \quad+ 2 \cdot \langle h(w^{(t)};i) - h_i^* , \epsilon_i^{(t)} \rangle  \\
    & \quad- 2 \cdot\langle h(w^{(t)};i) - h_i^* , \eta^{(t)} H_i^{(t)} v_{*\ \text{reg}}^{(t)} - \alpha_i^{(t)} \nabla_z \phi_i (h(w^{(t)};i))  \rangle \\
    & \quad + 2 \cdot \langle \eta^{(t)} H_i^{(t)} (v^{(t)} - v_{*\ \text{reg}}^{(t)}) ,  \alpha_i^{(t)} \nabla_z \phi_i (h(w^{(t)};i)) \rangle \\
    & \quad - 2 \cdot\langle \eta^{(t)} H_i^{(t)} (v^{(t)} - v_{*\ \text{reg}}^{(t)}) , \epsilon_i^{(t)} \rangle \\
    & \quad + 2 \cdot \langle \eta^{(t)} H_i^{(t)} (v^{(t)} - v_{*\ \text{reg}}^{(t)}), \eta^{(t)} H_i^{(t)} v_{*\ \text{reg}}^{(t)} - \alpha_i^{(t)} \nabla_z \phi_i (h(w^{(t)};i)) \rangle \\
    & \quad- 2 \cdot \langle \alpha_i^{(t)} \nabla_z \phi_i (h(w^{(t)};i)) , \epsilon_i^{(t)} \rangle \\
    & \quad + 2 \cdot \langle \alpha_i^{(t)} \nabla_z \phi_i (h(w^{(t)};i)), \eta^{(t)} H_i^{(t)} v_{*\ \text{reg}}^{(t)} - \alpha_i^{(t)} \nabla_z \phi_i (h(w^{(t)};i))  \rangle \\
    & \quad - 2 \cdot\langle \epsilon_i^{(t)}, \eta^{(t)} H_i^{(t)} v_{*\ \text{reg}}^{(t)} - \alpha_i^{(t)} \nabla_z \phi_i (h(w^{(t)};i))  \rangle,
\end{align*}
where we expand the square term. Now applying Young's inequalities: $2| \langle u , v \rangle |  \leq \frac{\| u \|^2}{\varepsilon/3} + (\varepsilon/3) \| v \|^2$ for $\varepsilon > 0$ and $2| \langle u , v \rangle | \leq \| u \|^2 + \| v \|^2$ we have:
\begin{align*}
    &\qquad \| h(w^{(t+1)};i) - h_i^* \|^2\\
    &= \| h(w^{(t)};i) - h_i^* \|^2 +  \|\eta^{(t)} H_i^{(t)} (v^{(t)} - v_{*\ \text{reg}}^{(t)})  \|^2  + ( \alpha_i^{(t)} )^2 \| \nabla_z \phi_i (h(w^{(t)};i)) \|^2  \\ & \quad + \| \epsilon_i^{(t)} \|^2+ \| \eta^{(t)} H_i^{(t)} v_{*\ \text{reg}}^{(t)} - \alpha_i^{(t)} \nabla_z \phi_i (h(w^{(t)};i)) \|^2 \\
   & \quad + \frac{\varepsilon}{3} \| h(w^{(t)};i) - h_i^* \|^2 + \frac{3}{\varepsilon} \| \eta^{(t)} H_i^{(t)} (v^{(t)} - v_{*\ \text{reg}}^{(t)}) \|^2 \\ 
    & \quad - 2 \alpha_i^{(t)} \langle  h(w^{(t)};i) - h_i^* , \nabla_z \phi_i (h(w^{(t)};i)) \rangle\\ 
    & \quad
    + \frac{\varepsilon}{3} \| h(w^{(t)};i) - h_i^* \|^2 +  \frac{3}{\varepsilon} \| \epsilon_i^{(t)} \|^2 \\
    & \quad + \frac{\varepsilon}{3} \| h(w^{(t)};i) - h_i^* \|^2 + \frac{3}{\varepsilon} \| \eta^{(t)} H_i^{(t)} v_{*\ \text{reg}}^{(t)} - \alpha_i^{(t)} \nabla_z \phi_i (h(w^{(t)};i))  \|^2\\
    & \quad + 3 ( \eta^{(t)} )^2 \| H_i^{(t)} (v^{(t)} - v_{*\ \text{reg}}^{(t)})  \|^2 + 3 ( \alpha_i^{(t)} )^2 \| \nabla_z \phi_i (h(w^{(t)};i)) \|^2 + 3 \| \epsilon_i^{(t)} \|^2 \\
    & \quad + 3 \| \eta^{(t)} H_i^{(t)} v_{*\ \text{reg}}^{(t)} - \alpha_i^{(t)} \nabla_z \phi_i (h(w^{(t)};i)) \|^2 \\
    & \overset{\eqref{eq_phi_convex}}{\leq} (1 + \varepsilon) \| h(w^{(t)};i) - h_i^* \|^2  + 4 ( \alpha_i^{(t)} )^2 \| \nabla_z \phi_i (h(w^{(t)};i)) \|^2\\
    & \quad +\left( 4+ \frac{3}{\varepsilon} \right)\| \eta^{(t)} H_i^{(t)} (v^{(t)} - v_{*\ \text{reg}}^{(t)}) \|^2 + \left( 4+ \frac{3}{\varepsilon} \right) \| \epsilon_i^{(t)} \|^2  \\
    & \quad + \left( 4+ \frac{3}{\varepsilon} \right) \| \eta^{(t)} H_i^{(t)} v_{*\ \text{reg}}^{(t)} - \alpha_i^{(t)} \nabla_z \phi_i (h(w^{(t)};i)) \|^2 \\ 
    & \quad - 2 \alpha_i^{(t)} [ \phi_i (h(w^{(t)};i) ) - \phi_i (h_i^* ) ] 
\end{align*}

Note that from \eqref{eq_phi_convex_smooth} we get that $\| \nabla_z \phi_i (h(w^{(t)};i)) \|^2 \leq 2 L_{\phi} [ \phi_i (h(w^{(t)};i)) - \phi_i (h_i^* ) ]$. Applying this and using the fact that $\alpha_i^{(t)} \leq \frac{\alpha}{L_{\phi}}$, for some $\alpha \in (0, \frac{1}{4})$, we are able to derive:
\begin{align*}
    &\qquad \| h(w^{(t+1)};i) - h_i^* \|^2 \\
    & \leq (1 + \varepsilon) \| h(w^{(t)};i) - h_i^* \|^2 - 2 ( 1 - 4 \alpha) \alpha_i^{(t)} [ \phi_i (h(w^{(t)};i)  ) - \phi_i (h_i^* ) ] \\
    & \quad +\frac{4\varepsilon+3}{\varepsilon}\| \eta^{(t)} H_i^{(t)} (v^{(t)} - v_{*\ \text{reg}}^{(t)}) \|^2 + \frac{4\varepsilon+3}{\varepsilon}\| \epsilon_i^{(t)} \|^2  \\
    & \quad + \frac{4\varepsilon+3}{\varepsilon} \| \eta^{(t)} H_i^{(t)} v_{*\ \text{reg}}^{(t)} - \alpha_i^{(t)} \nabla_z \phi_i (h(w^{(t)};i)) \|^2 \\ 
    & \overset{(a)}{\leq} (1 + \varepsilon) \| h(w^{(t)};i) - h_i^* \|^2 - 2 ( 1 - 4 \alpha) \alpha_i^{(t)} [ \phi_i (h(w^{(t)};i)  ) - \phi_i (h_i^* ) ]\\
    & \quad +\frac{4\varepsilon+3}{\varepsilon} D^2\varepsilon \frac{H^2}{\varepsilon} \| v^{(t)} - v_{*\ \text{reg}}^{(t)}  \|^2  + \frac{4\varepsilon+3}{\varepsilon}\| \epsilon_i^{(t)} \|^2  \\
    & \quad + \frac{4\varepsilon+3}{\varepsilon} \| \eta^{(t)} H_i^{(t)} v_{*\ \text{reg}}^{(t)} - \alpha_i^{(t)} \nabla_z \phi_i (h(w^{(t)};i)) \|^2 \\
    & \overset{(b)}{\leq} (1 + \varepsilon) \| h(w^{(t)};i) - h_i^* \|^2 - 2 ( 1 - 4 \alpha) \alpha_i^{(t)} [ \phi_i (h(w^{(t)};i)  ) - \phi_i (h_i^* ) ]\\
    & \quad +\frac{4\varepsilon+3}{\varepsilon} D^2 H^2 \cdot \varepsilon^2 + \frac{4\varepsilon+3}{\varepsilon} \cdot c(2+(V+\varepsilon^2 +2) G D^2)^2 \varepsilon^2  \\
    & \quad + \frac{4\varepsilon+3}{\varepsilon} \| \eta^{(t)} H_i^{(t)} v_{*\ \text{reg}}^{(t)} - \alpha_i^{(t)} \nabla_z \phi_i (h(w^{(t)};i)) \|^2 \\
    & = (1 + \varepsilon) \| h(w^{(t)};i) - h_i^* \|^2 - 2 ( 1 - 4 \alpha) \alpha_i^{(t)} [ \phi_i (h(w^{(t)};i)  ) - \phi_i (h_i^* ) ]\\
    & \quad +  \varepsilon(4\varepsilon+3)\left[D^2H^2 + c (2+(V+\varepsilon^2 +2) G D^2)^2\right] \\
    & \quad + \frac{4\varepsilon+3}{\varepsilon} \| \eta^{(t)} H_i^{(t)} v_{*\ \text{reg}}^{(t)} - \alpha_i^{(t)} \nabla_z \phi_i (h(w^{(t)};i)) \|^2 
\end{align*}

where
 $(a)$ follows by using matrix vector inequality $\| H v \| \leq \| H \| \|v \| $, where $H \in \mathbb{R}^{c \times d}$ and $v \in \mathbb{R}^d$ and Assumption~\ref{ass_bounded_hessian.2} in \eqref{eq_ass_bounded_hessian.2} and $\eta^{(t)} = D\sqrt{\varepsilon}$ for some $D > 0$ and $\varepsilon > 0$; $(b)$ follows by the fact that $\| v^{(t)} - v_{*\ \text{reg}}^{(t)} \|^2 \leq \varepsilon^2$ in  \eqref{criteria_eps_sol.3} and Lemma \ref{lem_bound_epsilon.3}. 
\end{proof}

\begin{lemma}\label{lem_main_result_02_new.3}
Let $w^{(t)}$ be generated by Algorithm \ref{alg_GD_solved.2}
where $v^{(t)}$ satisfies \eqref{criteria_eps_sol.3}.
 We execute Algorithm \ref{alg_GD_solved.2} for $T = \frac{\beta}{\varepsilon}$ outer loops for some constant $\beta>0$.
We assume Assumption \ref{ass_phi} holds.
Suppose that Assumption \ref{ass_approx_h} holds for $G > 0$, Assumption \ref{ass_solve_v} holds for $V > 0$
and Assumption \ref{ass_bounded_hessian.2} holds for $H > 0$.
We set the step size equal to $\eta^{(t)} = D\sqrt{\varepsilon}$ for some $D > 0$ and choose a learning rate 
$ \alpha_i^{(t)} = ( 1 + \varepsilon ) \alpha_i^{(t-1)} = ( 1 + \varepsilon )^t \alpha_i^{(0)}$. Based on $\beta$, we define 
$\alpha_i^{(0)} = \frac{\alpha}{e^{\beta} L_{\phi}}$
with $\alpha \in (0, \frac{1}{4})$.
We have
\allowdisplaybreaks
\begin{align*}
    \frac{1}{T} \sum_{t=0}^{T-1} \frac{1}{n} \sum_{i=1}^n [ f (w^{(t)}; i) - \phi_i (h_i^* )  ] 
    & \leq  \frac{e^{\beta} L_{\phi}( 1 + \varepsilon )}{2 ( 1 - 4 \alpha) \alpha \beta} \cdot  \frac{1}{n} \sum_{i=1}^n  \| h(w^{(0)};i) - h_i^* \|^2 \cdot \varepsilon\\
    & \quad + \frac{e^{\beta} L_{\phi}(4\varepsilon+3)}{2\alpha( 1 - 4 \alpha)}   \left[D^2H^2 + c (2+(V+\varepsilon^2 +2) G D^2)^2 + 2+ V\right] \cdot \varepsilon. \tagthis \label{eq_lem_main_result_02_new.2}
\end{align*}
\end{lemma}

\begin{proof}
Rearranging the terms in Lemma~\ref{lem_main_result_01_new.3},
we have
\allowdisplaybreaks
\begin{align*}
     \phi_i (h(w^{(t)};i) ) - \phi_i (h_i^* ) & \leq \frac{1}{2 ( 1 - 4 \alpha)}\left( \frac{(1 + \varepsilon)}{\alpha_i^{(t)}} \| h(w^{(t)};i) - h_i^* \|^2 - \frac{1}{\alpha_i^{(t)}} \| h(w^{(t+1)};i) - h_i^* \|^2 \right) \\
    & \quad + \frac{1}{2( 1 - 4 \alpha)} \cdot \frac{1}{\alpha_i^{(t)}} \cdot \varepsilon(4\varepsilon+3) \left[D^2H^2 + c (2+(V+\varepsilon^2 +2) G D^2)^2\right] \\
    & \quad
    + \frac{1}{2( 1 - 4 \alpha)} \cdot \frac{1}{\alpha_i^{(t)}} \cdot\frac{4\varepsilon+3}{\varepsilon} \| \eta^{(t)} H_i^{(t)} v_{*\ \text{reg}}^{(t)} - \alpha_i^{(t)} \nabla_z \phi_i (h(w^{(t)};i)) \|^2   \\
    & \leq \frac{1}{2 ( 1 - 4 \alpha)}\left( \frac{(1 + \varepsilon)}{\alpha_i^{(t)}}  \| h(w^{(t)};i) - h_i^* \|^2 - \frac{(1 + \varepsilon)}{\alpha_i^{(t+1)}} \| h(w^{(t+1)};i) - h_i^* \|^2 \right) \\
    & \quad + \frac{e^{\beta} L_{\phi}}{2\alpha( 1 - 4 \alpha)}  \cdot \varepsilon(4\varepsilon+3) \left[D^2H^2 + c (2+(V+\varepsilon^2 +2) G D^2)^2\right] \\
    & \quad
    + \frac{e^{\beta} L_{\phi}}{2\alpha( 1 - 4 \alpha)}  \cdot\frac{4\varepsilon+3}{\varepsilon} \| \eta^{(t)} H_i^{(t)} v_{*\ \text{reg}}^{(t)} - \alpha_i^{(t)} \nabla_z \phi_i (h(w^{(t)};i)) \|^2 .
    \tagthis \label{eq_0001_new.2}
\end{align*}
The last inequality follows because the learning rate satisfies $\alpha_i^{(0)} = \frac{\alpha}{e^{\beta} L_{\phi}} \leq \frac{\alpha}{L_{\phi}}$ and for $t = 1,\dots,T = \frac{\beta}{\varepsilon}$ for some $\beta > 0$
\begin{align*}
    \alpha_i^{(t)} = ( 1 + \varepsilon ) \alpha_i^{(t-1)} = ( 1 + \varepsilon )^t \alpha_i^{(0)} \leq ( 1 + \varepsilon )^T \alpha_i^{(0)} = ( 1 + \varepsilon )^{\beta/\varepsilon} \frac{\alpha}{e^{\beta} L_{\phi}} \leq \frac{\alpha}{L_{\phi}}, 
\end{align*}
since $(1 + x)^{1/x} \leq e$, $x > 0$. Moreover, we have $\frac{1}{\alpha_i^{(t)}} \leq \frac{1}{\alpha_i^{(0)}} = \frac{e^{\beta} L_{\phi}}{\alpha}$, $t = 0,\dots,T-1$. 

Taking the average sum from $t = 0,\dots,T-1$, we have
\begin{align*}
    \frac{1}{T} \sum_{t=0}^{T-1} [\phi_i (h(w^{(t)};i) ) - \phi_i (h_i^* ) ] & \leq \frac{1}{2 ( 1 - 4 \alpha) T} \cdot \frac{( 1 + \varepsilon ) }{\alpha_i^{(0)}} \| h(w^{(0)};i) - h_i^* \|^2 \\
    & \quad + \frac{e^{\beta} L_{\phi}}{2\alpha( 1 - 4 \alpha)}  \cdot \varepsilon(4\varepsilon+3) \left[D^2H^2 + c (2+(V+\varepsilon^2 +2) G D^2)^2\right] \\
    & \quad
    + \frac{e^{\beta} L_{\phi}}{2\alpha( 1 - 4 \alpha)}  \cdot\frac{4\varepsilon+3}{\varepsilon} \frac{1}{T} \sum_{t=0}^{T-1} \| \eta^{(t)} H_i^{(t)} v_{*\ \text{reg}}^{(t)} - \alpha_i^{(t)} \nabla_z \phi_i (h(w^{(t)};i)) \|^2 \\
    & = \frac{e^{\beta} L_{\phi}( 1 + \varepsilon )}{2 ( 1 - 4 \alpha) \alpha \beta} \varepsilon \cdot \| h(w^{(0)};i) - h_i^* \|^2 \\
    & \quad + \frac{e^{\beta} L_{\phi}}{2\alpha( 1 - 4 \alpha)}  \cdot \varepsilon(4\varepsilon+3) \left[D^2H^2 + c (2+(V+\varepsilon^2 +2) G D^2)^2\right] \\
    & \quad
    + \frac{e^{\beta} L_{\phi}}{2\alpha( 1 - 4 \alpha)}  \cdot\frac{4\varepsilon+3}{\varepsilon} \frac{1}{T} \sum_{t=0}^{T-1} \| \eta^{(t)} H_i^{(t)} v_{*\ \text{reg}}^{(t)} - \alpha_i^{(t)} \nabla_z \phi_i (h(w^{(t)};i)) \|^2.
\end{align*}
Taking the average sum from $i = 1,\dots,n$, we have
\begin{align*}
    & \frac{1}{T} \sum_{t=0}^{T-1} \frac{1}{n} \sum_{i=1}^n [ \phi_i (h(w^{(t)};i) ) - \phi_i (h_i^* ) ]  \\ 
    & \leq \frac{e^{\beta} L_{\phi}( 1 + \varepsilon )}{2 ( 1 - 4 \alpha) \alpha \beta} \varepsilon \cdot  \frac{1}{n} \sum_{i=1}^n  \| h(w^{(0)};i) - h_i^* \|^2 \\
    & \quad + \frac{e^{\beta} L_{\phi}}{2\alpha( 1 - 4 \alpha)}  \cdot \varepsilon(4\varepsilon+3) \left[D^2H^2 + c (2+(V+\varepsilon^2 +2) G D^2)^2\right] \\
    & \quad
    + \frac{e^{\beta} L_{\phi}}{2\alpha( 1 - 4 \alpha)}  \cdot\frac{4\varepsilon+3}{\varepsilon} \frac{1}{T} \sum_{t=0}^{T-1} \frac{1}{n} \sum_{i=1}^n  \| \eta^{(t)} H_i^{(t)} v_{*\ \text{reg}}^{(t)} - \alpha_i^{(t)} \nabla_z \phi_i (h(w^{(t)};i)) \|^2\\
    & \overset{\eqref{eq_lem_stronglyconvex_prob}}{\leq}  \frac{e^{\beta} L_{\phi}( 1 + \varepsilon )}{2 ( 1 - 4 \alpha) \alpha \beta} \varepsilon \cdot  \frac{1}{n} \sum_{i=1}^n  \| h(w^{(0)};i) - h_i^* \|^2 \\
    & \quad + \frac{e^{\beta} L_{\phi}}{2\alpha( 1 - 4 \alpha)}  \cdot \varepsilon(4\varepsilon+3) \left[D^2H^2 + c (2+(V+\varepsilon^2 +2) G D^2)^2\right] \\
    & \quad
    + \frac{e^{\beta} L_{\phi}}{2\alpha( 1 - 4 \alpha)}  \cdot\frac{4\varepsilon+3}{\varepsilon} (2+V) \varepsilon^2. \tagthis \label{eq_0002_new.2}
\end{align*}

Note that
\begin{align*}
     \frac{1}{T} \sum_{t=0}^{T-1} \frac{1}{n} \sum_{i=1}^n [ \phi_i (h(w^{(t)};i) ) - \phi_i (h_i^* ) ]  = \frac{1}{T} \sum_{t=0}^{T-1} \frac{1}{n} \sum_{i=1}^n [ f (w^{(t)}; i) - \phi_i (h_i^* )  ]. \tagthis \label{eq_average_phi_02.2}
\end{align*}

Therefore, applying \eqref{eq_average_phi_02.2} to \eqref{eq_0002_new.2}, we have
\begin{align*}
     \frac{1}{T} \sum_{t=0}^{T-1} \frac{1}{n} \sum_{i=1}^n [ f (w^{(t)}; i) - \phi_i (h_i^* )  ]  
    & \leq  \frac{e^{\beta} L_{\phi}( 1 + \varepsilon )}{2 ( 1 - 4 \alpha) \alpha \beta}  \cdot  \frac{1}{n} \sum_{i=1}^n  \| h(w^{(0)};i) - h_i^* \|^2 \cdot \varepsilon\\
    & \quad + \frac{e^{\beta} L_{\phi}}{2\alpha( 1 - 4 \alpha)}  (4\varepsilon+3) \left[D^2H^2 + c (2+(V+\varepsilon^2 +2) G D^2)^2 + 2+ V\right]\cdot \varepsilon . 
\end{align*}
\end{proof}

\subsection*{Proof of Theorem~\ref{thm_main_result_GD_solved.3}}





\begin{proof}
From \eqref{eq_loss_great_01} we have $F_* - \frac{1}{n} \sum_{i=1}^n \phi_i (h_i^* ) \geq 0$. This leads to
\allowdisplaybreaks
\begin{align*}
    \frac{1}{T} \sum_{t=0}^{T-1} [ F (w^{(t)}) - F_* ] & = \frac{1}{T} \sum_{t=0}^{T-1} \left( \frac{1}{n} \sum_{i=1}^n [ f (w^{(t)}; i) - \phi_i (h_i^* )  ] - \left[ F_* - \frac{1}{n} \sum_{i=1}^{n} \phi_i (h_i^* )  \right] \right) \\
    & \overset{\eqref{eq_loss_great_01}}{\leq} \frac{1}{T} \sum_{t=0}^{T-1} \frac{1}{n} \sum_{i=1}^n [ f (w^{(t)}; i) - \phi_i (h_i^* )  ] \\
    & \overset{\eqref{eq_lem_main_result_02_new.2}
    }{\leq}  \frac{e^{\beta} L_{\phi}( 1 + \varepsilon )}{2 ( 1 - 4 \alpha) \alpha \beta} \cdot  \frac{1}{n} \sum_{i=1}^n  \| h(w^{(0)};i) - h_i^* \|^2 \cdot \varepsilon \\
    & \quad + \frac{e^{\beta} L_{\phi}(4\varepsilon+3)}{2\alpha( 1 - 4 \alpha)}  \left[D^2H^2 + c (2+(V+\varepsilon^2 +2) G D^2)^2 + 2+ V\right] \cdot \varepsilon. \tagthis \label{eq_thm_main_result_GD_solved_new.2}
\end{align*}

\end{proof}

\subsection*{Proof of Corollary~\ref{cor_thm_main_result_GD_solved_new.3}}
\begin{proof} 
For each iteration $0 \leq t < T$, we need to find $v^{(t)}$ satisfying the following criteria: 
\begin{align*}
     \| v^{(t)} - v_{*\ \text{reg}}^{(t)} \|^2 \leq \varepsilon^2, 
\end{align*}
for some $\varepsilon > 0$. Using Gradient Descent we need $\mathcal{O}(n \frac{L}{\mu} \log(\frac{1}{\varepsilon^2})) = \mathcal{O}(2 n \frac{L}{\mu} \log(\frac{1}{\varepsilon}))$ number of gradient evaluations \citep{nesterov2004}, where $L$ and $\mu = \varepsilon^2$ are the smooth and strongly convex constants, respectively, of $\Psi$. Let 
\begin{align*}
    \psi_i^{(t)} (v) = \frac{1}{2} \| \eta^{(t)} H_i^{(t)} v - \alpha_i^{(t)} \nabla_z \phi_i (h(w^{(t)};i))  \|^2, \ i \in [n]. \tagthis \label{eq_func_v.2}
\end{align*}
Then, for any $v \in \mathbb{R}^c$
\begin{align*}
    \nabla_v \psi_i^{(t)} (v) = \eta^{(t)} H_i^{(t)}{^\top} [ \eta^{(t)} H_i^{(t)} v - \alpha_i^{(t)} \nabla_z \phi_i (h(w^{(t)};i)) ], \ i \in [n]. \tagthis \label{eq_grad_v.2}
\end{align*}
Consider $\eta^{(t)} = D\sqrt {\varepsilon}$ for some $D > 0$ and $\varepsilon > 0$, we have for $i \in [n]$ and $0 \leq t < T$
\begin{align*}
    \| \nabla^2_v \psi_i^{(t)} ( v ) \| = (\eta^{(t)})^2 \| H_i^{(t)}{^\top} H_i^{(t)} \| \leq (\eta^{(t)})^2 \| H_i^{(t)} \| \cdot \| H_i^{(t)} \| \overset{\eqref{eq_ass_bounded_hessian.2}}{\leq} D^2 H^2. 
\end{align*}

Hence, $\| \nabla^2_v \Psi^{(t)}(v) \| \leq D^2 H^2 + \varepsilon^2 $ for any $v \in \mathbb{R}^c$ which implies that $L = D^2 H^2 + \varepsilon^2$ (\cite{nesterov2004}) and $\frac{L}{\mu} = \frac{D^2H^2  + \varepsilon^2}{\varepsilon^2}$. Therefore, the complexity to find $v^{(t)}$ for each iteration $t$ is $\mathcal{O}(2n \frac{D^2H^2 + \varepsilon^2}{\varepsilon^2}\log(\frac{1}{\varepsilon}))$.  

Let us choose $0 < \varepsilon \leq 1$. From \eqref{eq_thm_main_result_GD_solved_new.2}, we have
\begin{align*}
    \frac{1}{T} \sum_{t=0}^{T-1} [ F (w^{(t)}) - F_* ] & \leq  \frac{ e^{\beta} L_{\phi}}{ ( 1 - 4 \alpha) \alpha \beta} \cdot  \frac{1}{n} \sum_{i=1}^n  \| h(w^{(0)};i) - h_i^* \|^2 \cdot \varepsilon \\
    & \quad + \frac{7 e^{\beta} L_{\phi}}{2\alpha( 1 - 4 \alpha)}  \left[D^2H^2 + c (2+(V+3) G D^2)^2 + 2+ V\right] \cdot \varepsilon  = N \varepsilon, 
\end{align*}
where 
\begin{align*}
    N = \tfrac{ e^{\beta} L_{\phi}}{( 1 - 4 \alpha) \alpha \beta}  \frac{1}{n} \sum_{i=1}^n  \| h(w^{(0)};i) - h_i^* \|^2 + \tfrac{7 e^{\beta} L_{\phi}}{2\alpha( 1 - 4 \alpha)}  \left[D^2H^2 + c (2+(V+3) G D^2)^2 + 2+ V\right]. 
\end{align*}

Let $\hat{\varepsilon} = N \varepsilon$ with $0< \hat{\varepsilon} \leq N$. Then, we need $T = \frac{N \beta}{\hat{\varepsilon}}$ for some $\beta > 0$ to guarantee $\min_{0\leq t\leq T-1} [ F (w^{(t)}) - F_* ] \leq \frac{1}{T} \sum_{t=0}^{T-1} [ F (w^{(t)}) - F_* ] \leq \hat{\varepsilon}$. Hence, the total complexity is 
$\mathcal{O}\left(n \frac{N^3 \beta}{\hat{\varepsilon}^3} (D^2H^2 + (\hat{\varepsilon}^2/N))\log(\frac{N}{\hat{\varepsilon}}) \right)$.
\end{proof}

\bibliographystyle{plainnat}
\bibliography{reference}

\begin{thebibliography}{36}
\providecommand{\natexlab}[1]{#1}
\providecommand{\url}[1]{\texttt{#1}}
\expandafter\ifx\csname urlstyle\endcsname\relax
  \providecommand{\doi}[1]{doi: #1}\else
  \providecommand{\doi}{doi: \begingroup \urlstyle{rm}\Url}\fi

\bibitem[Allen-Zhu et~al.(2019)Allen-Zhu, Li, and Song]{pmlr-v97-allen-zhu19a}
Zeyuan Allen-Zhu, Yuanzhi Li, and Zhao Song.
\newblock A convergence theory for deep learning via over-parameterization.
\newblock In Kamalika Chaudhuri and Ruslan Salakhutdinov, editors,
  \emph{Proceedings of the 36th International Conference on Machine Learning},
  volume~97 of \emph{Proceedings of Machine Learning Research}, pages 242--252.
  PMLR, 09--15 Jun 2019.
\newblock URL \url{http://proceedings.mlr.press/v97/allen-zhu19a.html}.

\bibitem[Arora et~al.(2018)Arora, Cohen, and Hazan]{pmlr-v80-arora18a}
Sanjeev Arora, Nadav Cohen, and Elad Hazan.
\newblock On the optimization of deep networks: Implicit acceleration by
  overparameterization.
\newblock In Jennifer Dy and Andreas Krause, editors, \emph{Proceedings of the
  35th International Conference on Machine Learning}, volume~80 of
  \emph{Proceedings of Machine Learning Research}, pages 244--253. PMLR, 10--15
  Jul 2018.
\newblock URL \url{http://proceedings.mlr.press/v80/arora18a.html}.

\bibitem[Arora et~al.(2019)Arora, Du, Hu, Li, and Wang]{pmlr-v97-arora19a}
Sanjeev Arora, Simon Du, Wei Hu, Zhiyuan Li, and Ruosong Wang.
\newblock Fine-grained analysis of optimization and generalization for
  overparameterized two-layer neural networks.
\newblock In Kamalika Chaudhuri and Ruslan Salakhutdinov, editors,
  \emph{Proceedings of the 36th International Conference on Machine Learning},
  volume~97 of \emph{Proceedings of Machine Learning Research}, pages 322--332.
  PMLR, 09--15 Jun 2019.
\newblock URL \url{http://proceedings.mlr.press/v97/arora19a.html}.

\bibitem[Bottou et~al.(2018)Bottou, Curtis, and Nocedal]{bottou_survey}
Léon Bottou, Frank~E. Curtis, and Jorge Nocedal.
\newblock Optimization methods for large-scale machine learning.
\newblock \emph{SIAM Review}, 60\penalty0 (2):\penalty0 223--311, 2018.
\newblock \doi{10.1137/16M1080173}.

\bibitem[Brutzkus et~al.(2018)Brutzkus, Globerson, Malach, and
  Shalev-Shwartz]{brutzkus2017sgd}
Alon Brutzkus, Amir Globerson, Eran Malach, and Shai Shalev-Shwartz.
\newblock {SGD} learns over-parameterized networks that provably generalize on
  linearly separable data.
\newblock In \emph{International Conference on Learning Representations}, 2018.
\newblock URL \url{https://openreview.net/forum?id=rJ33wwxRb}.

\bibitem[Defazio et~al.(2014)Defazio, Bach, and Lacoste-Julien]{SAGA}
Aaron Defazio, Francis Bach, and Simon Lacoste-Julien.
\newblock Saga: A fast incremental gradient method with support for
  non-strongly convex composite objectives.
\newblock In \emph{Advances in Neural Information Processing Systems}, pages
  1646--1654, 2014.

\bibitem[Du et~al.(2019{\natexlab{a}})Du, Lee, Li, Wang, and
  Zhai]{pmlr-v97-du19c}
Simon Du, Jason Lee, Haochuan Li, Liwei Wang, and Xiyu Zhai.
\newblock Gradient descent finds global minima of deep neural networks.
\newblock In Kamalika Chaudhuri and Ruslan Salakhutdinov, editors,
  \emph{Proceedings of the 36th International Conference on Machine Learning},
  volume~97 of \emph{Proceedings of Machine Learning Research}, pages
  1675--1685. PMLR, 09--15 Jun 2019{\natexlab{a}}.
\newblock URL \url{http://proceedings.mlr.press/v97/du19c.html}.

\bibitem[Du et~al.(2019{\natexlab{b}})Du, Zhai, Poczos, and
  Singh]{du2019gradient}
Simon~S. Du, Xiyu Zhai, Barnabas Poczos, and Aarti Singh.
\newblock Gradient descent provably optimizes over-parameterized neural
  networks.
\newblock In \emph{International Conference on Learning Representations},
  2019{\natexlab{b}}.
\newblock URL \url{https://openreview.net/forum?id=S1eK3i09YQ}.

\bibitem[Duchi et~al.(2011)Duchi, Hazan, and Singer]{2011duchi11a}
John Duchi, Elad Hazan, and Yoram Singer.
\newblock Adaptive subgradient methods for online learning and stochastic
  optimization.
\newblock \emph{Journal of Machine Learning Research}, 12\penalty0
  (61):\penalty0 2121--2159, 2011.
\newblock URL \url{http://jmlr.org/papers/v12/duchi11a.html}.

\bibitem[Ghadimi and Lan(2013)]{ghadimi2013stochastic}
S.~Ghadimi and G.~Lan.
\newblock Stochastic first-and zeroth-order methods for nonconvex stochastic
  programming.
\newblock \emph{SIAM J. Optim.}, 23\penalty0 (4):\penalty0 2341--2368, 2013.

\bibitem[Gower et~al.(2021)Gower, Sebbouh, and Loizou]{pmlr-v130-gower21a}
Robert Gower, Othmane Sebbouh, and Nicolas Loizou.
\newblock Sgd for structured nonconvex functions: Learning rates, minibatching
  and interpolation.
\newblock In Arindam Banerjee and Kenji Fukumizu, editors, \emph{Proceedings of
  The 24th International Conference on Artificial Intelligence and Statistics},
  volume 130 of \emph{Proceedings of Machine Learning Research}, pages
  1315--1323. PMLR, 13--15 Apr 2021.
\newblock URL \url{https://proceedings.mlr.press/v130/gower21a.html}.

\bibitem[Gürbüzbalaban et~al.(2019)Gürbüzbalaban, Ozdaglar, and
  Parrilo]{2019gurbuzbalaban}
M.~Gürbüzbalaban, A.~Ozdaglar, and P.~A. Parrilo.
\newblock Convergence rate of incremental gradient and incremental newton
  methods.
\newblock \emph{SIAM Journal on Optimization}, 29\penalty0 (4):\penalty0
  2542--2565, 2019.
\newblock \doi{10.1137/17M1147846}.
\newblock URL \url{https://doi.org/10.1137/17M1147846}.

\bibitem[Johnson and Zhang(2013)]{SVRG}
Rie Johnson and Tong Zhang.
\newblock Accelerating stochastic gradient descent using predictive variance
  reduction.
\newblock In \emph{Advances in Neural Information Processing Systems}, pages
  315--323, 2013.

\bibitem[Karimi et~al.(2016)Karimi, Nutini, and Schmidt]{polyak_condition}
Hamed Karimi, Julie Nutini, and Mark Schmidt.
\newblock Linear convergence of gradient and proximal-gradient methods under
  the {P}olyak-{{\L}}ojasiewicz condition.
\newblock In Paolo Frasconi, Niels Landwehr, Giuseppe Manco, and Jilles
  Vreeken, editors, \emph{Machine Learning and Knowledge Discovery in
  Databases}, pages 795--811, Cham, 2016. Springer International Publishing.

\bibitem[Kingma and Ba(2014)]{KingmaB14}
Diederik~P. Kingma and Jimmy Ba.
\newblock Adam: A method for stochastic optimization.
\newblock \emph{CoRR}, abs/1412.6980, 2014.

\bibitem[Le~Roux et~al.(2012)Le~Roux, Schmidt, and Bach]{SAG}
Nicolas Le~Roux, Mark Schmidt, and Francis Bach.
\newblock A stochastic gradient method with an exponential convergence rate for
  finite training sets.
\newblock In \emph{NIPS}, pages 2663--2671, 2012.

\bibitem[Levy et~al.(2018)Levy, Yurtsever, and Cevher]{2018Levy}
Kfir~Y. Levy, Alp Yurtsever, and Volkan Cevher.
\newblock Online adaptive methods, universality and acceleration.
\newblock In S.~Bengio, H.~Wallach, H.~Larochelle, K.~Grauman, N.~Cesa-Bianchi,
  and R.~Garnett, editors, \emph{Advances in Neural Information Processing
  Systems}, volume~31. Curran Associates, Inc., 2018.
\newblock URL
  \url{https://proceedings.neurips.cc/paper/2018/file/b0169350cd35566c47ba83c6ec1d6f82-Paper.pdf}.

\bibitem[Lewis and Wright(2016)]{Lewis2016APM}
Adrian~S. Lewis and Stephen~J. Wright.
\newblock A proximal method for composite minimization.
\newblock \emph{Mathematical Programming}, 158:\penalty0 501--546, 2016.

\bibitem[Nemirovski et~al.(2009)Nemirovski, Juditsky, Lan, and
  Shapiro]{Nemirovski2009}
A.~Nemirovski, A.~Juditsky, G.~Lan, and A.~Shapiro.
\newblock Robust stochastic approximation approach to stochastic programming.
\newblock \emph{SIAM J. on Optimization}, 19\penalty0 (4):\penalty0 1574--1609,
  2009.

\bibitem[Nesterov(2004)]{nesterov2004}
Yurii Nesterov.
\newblock \emph{Introductory lectures on convex optimization : a basic course}.
\newblock Applied optimization. Kluwer Academic Publ., Boston, Dordrecht,
  London, 2004.
\newblock ISBN 1-4020-7553-7.

\bibitem[Nguyen et~al.(2018)Nguyen, Nguyen, van Dijk, Richtarik, Scheinberg,
  and Takac]{Nguyen2018_sgdhogwild}
Lam Nguyen, Phuong~Ha Nguyen, Marten van Dijk, Peter Richtarik, Katya
  Scheinberg, and Martin Takac.
\newblock {SGD} and {H}ogwild! convergence without the bounded gradients
  assumption.
\newblock In \emph{Proceedings of the 35th International Conference on Machine
  Learning-Volume 80}, pages 3747--3755, 2018.

\bibitem[Nguyen et~al.(2017)Nguyen, Liu, Scheinberg, and
  Tak{\'a}{\v{c}}]{Nguyen2017sarah}
Lam~M Nguyen, Jie Liu, Katya Scheinberg, and Martin Tak{\'a}{\v{c}}.
\newblock Sarah: A novel method for machine learning problems using stochastic
  recursive gradient.
\newblock In \emph{Proceedings of the 34th International Conference on Machine
  Learning-Volume 70}, pages 2613--2621. JMLR. org, 2017.

\bibitem[Nguyen et~al.(2019)Nguyen, Nguyen, Richt{{\'a}}rik, Scheinberg,
  Tak{{\'a}}{\v{c}}, and van Dijk]{Nguyen2019_sgd_new_aspects}
Lam~M. Nguyen, Phuong~Ha Nguyen, Peter Richt{{\'a}}rik, Katya Scheinberg,
  Martin Tak{{\'a}}{\v{c}}, and Marten van Dijk.
\newblock New convergence aspects of stochastic gradient algorithms.
\newblock \emph{Journal of Machine Learning Research}, 20\penalty0
  (176):\penalty0 1--49, 2019.
\newblock URL \url{http://jmlr.org/papers/v20/18-759.html}.

\bibitem[Nguyen et~al.(2021)Nguyen, Tran-Dinh, Phan, Nguyen, and van
  Dijk]{nguyen2020unified}
Lam~M. Nguyen, Quoc Tran-Dinh, Dzung~T. Phan, Phuong~Ha Nguyen, and Marten van
  Dijk.
\newblock A unified convergence analysis for shuffling-type gradient methods.
\newblock \emph{Journal of Machine Learning Research}, 22\penalty0
  (207):\penalty0 1--44, 2021.

\bibitem[Nguyen and Mondelli(2020)]{nguyen2020global}
Quynh~N Nguyen and Marco Mondelli.
\newblock Global convergence of deep networks with one wide layer followed by
  pyramidal topology.
\newblock In H.~Larochelle, M.~Ranzato, R.~Hadsell, M.~F. Balcan, and H.~Lin,
  editors, \emph{Advances in Neural Information Processing Systems}, volume~33,
  pages 11961--11972. Curran Associates, Inc., 2020.
\newblock URL
  \url{https://proceedings.neurips.cc/paper/2020/file/8abfe8ac9ec214d68541fcb888c0b4c3-Paper.pdf}.

\bibitem[Reddi et~al.(2016)Reddi, Hefny, Sra, Poczos, and
  Smola]{svrg_nonconvex-reddi16}
Sashank~J. Reddi, Ahmed Hefny, Suvrit Sra, Barnabas Poczos, and Alex Smola.
\newblock Stochastic variance reduction for nonconvex optimization.
\newblock In Maria~Florina Balcan and Kilian~Q. Weinberger, editors,
  \emph{Proceedings of The 33rd International Conference on Machine Learning},
  volume~48 of \emph{Proceedings of Machine Learning Research}, pages 314--323,
  New York, New York, USA, 20--22 Jun 2016. PMLR.
\newblock URL \url{https://proceedings.mlr.press/v48/reddi16.html}.

\bibitem[Reddi et~al.(2018)Reddi, Kale, and Kumar]{reddi2019}
Sashank~J. Reddi, Satyen Kale, and Sanjiv Kumar.
\newblock On the convergence of adam and beyond.
\newblock In \emph{International Conference on Learning Representations}, 2018.
\newblock URL \url{https://openreview.net/forum?id=ryQu7f-RZ}.

\bibitem[Shalev-Shwartz et~al.(2007)Shalev-Shwartz, Singer, and
  Srebro]{pegasos}
Shai Shalev-Shwartz, Yoram Singer, and Nathan Srebro.
\newblock Pegasos: Primal estimated sub-gradient solver for svm.
\newblock \emph{Association for Computing Machinery}, 2007.
\newblock \doi{10.1145/1273496.1273598}.
\newblock URL \url{https://doi.org/10.1145/1273496.1273598}.

\bibitem[Soudry et~al.(2018)Soudry, Hoffer, Nacson, Gunasekar, and
  Srebro]{soudry2018implicit}
Daniel Soudry, Elad Hoffer, Mor~Shpigel Nacson, Suriya Gunasekar, and Nathan
  Srebro.
\newblock The implicit bias of gradient descent on separable data.
\newblock \emph{J. Mach. Learn. Res.}, 19\penalty0 (1):\penalty0 2822–2878,
  January 2018.
\newblock ISSN 1532-4435.

\bibitem[Tran et~al.(2021)Tran, Nguyen, and Tran-Dinh]{smg_tran21b}
Trang~H Tran, Lam~M Nguyen, and Quoc Tran-Dinh.
\newblock {SMG}: A shuffling gradient-based method with momentum.
\newblock In Marina Meila and Tong Zhang, editors, \emph{Proceedings of the
  38th International Conference on Machine Learning}, volume 139 of
  \emph{Proceedings of Machine Learning Research}, pages 10379--10389. PMLR,
  18--24 Jul 2021.
\newblock URL \url{https://proceedings.mlr.press/v139/tran21b.html}.

\bibitem[Tran-Dinh et~al.(2020)Tran-Dinh, Pham, and
  Nguyen]{pmlr-v119-tran-dinh20a}
Quoc Tran-Dinh, Nhan Pham, and Lam Nguyen.
\newblock Stochastic {G}auss-{N}ewton algorithms for nonconvex compositional
  optimization.
\newblock In Hal~Daumé III and Aarti Singh, editors, \emph{Proceedings of the
  37th International Conference on Machine Learning}, volume 119 of
  \emph{Proceedings of Machine Learning Research}, pages 9572--9582. PMLR,
  13--18 Jul 2020.
\newblock URL \url{https://proceedings.mlr.press/v119/tran-dinh20a.html}.

\bibitem[Vaswani et~al.(2021)Vaswani, Laradji, Kunstner, Meng, Schmidt, and
  Lacoste-Julien]{vaswani2021adaptive}
Sharan Vaswani, Issam Laradji, Frederik Kunstner, Si~Yi Meng, Mark Schmidt, and
  Simon Lacoste-Julien.
\newblock Adaptive gradient methods converge faster with over-parameterization
  (but you should do a line-search), 2021.

\bibitem[Zhang and Xiao(2019)]{zhang2019stochastic}
Junyu Zhang and Lin Xiao.
\newblock A stochastic composite gradient method with incremental variance
  reduction.
\newblock In H.~Wallach, H.~Larochelle, A.~Beygelzimer, F.~d\textquotesingle
  Alch\'{e}-Buc, E.~Fox, and R.~Garnett, editors, \emph{Advances in Neural
  Information Processing Systems}, volume~32. Curran Associates, Inc., 2019.
\newblock URL
  \url{https://proceedings.neurips.cc/paper/2019/file/a68259547f3d25ab3c0a5c0adb4e3498-Paper.pdf}.

\bibitem[Zhang and Xiao(2021)]{zhang2021multilevel}
Junyu Zhang and Lin Xiao.
\newblock Multilevel composite stochastic optimization via nested variance
  reduction.
\newblock \emph{SIAM Journal on Optimization}, 31\penalty0 (2):\penalty0
  1131--1157, 2021.
\newblock \doi{10.1137/19M1285457}.
\newblock URL \url{https://doi.org/10.1137/19M1285457}.

\bibitem[Zou and Gu(2019)]{ZouG19nips}
Difan Zou and Quanquan Gu.
\newblock An improved analysis of training over-parameterized deep neural
  networks.
\newblock In Hanna~M. Wallach, Hugo Larochelle, Alina Beygelzimer, Florence
  d'Alch{\'{e}}{-}Buc, Emily~B. Fox, and Roman Garnett, editors, \emph{Advances
  in Neural Information Processing Systems 32: Annual Conference on Neural
  Information Processing Systems 2019, NeurIPS 2019, December 8-14, 2019,
  Vancouver, BC, Canada}, pages 2053--2062, 2019.

\bibitem[Zou et~al.(2018)Zou, Cao, Zhou, and Gu]{zou2018stochastic}
Difan Zou, Yuan Cao, Dongruo Zhou, and Quanquan Gu.
\newblock Stochastic gradient descent optimizes over-parameterized deep relu
  networks, 2018.

\end{thebibliography}





\end{document}